\documentclass[aos]{myimsart}

\RequirePackage{amsthm,amsmath,amsfonts,amssymb}
\RequirePackage[authoryear]{natbib}

\RequirePackage[colorlinks,citecolor=blue,urlcolor=blue]{hyperref}
\RequirePackage{graphicx}
\RequirePackage{url}
\RequirePackage{tikz}
\RequirePackage{algorithm2e}

\startlocaldefs

\usepackage{thmtools}
\usepackage{thm-restate}

\theoremstyle{plain}
\newtheorem{corollary}{Corollary}
\newtheorem{lemma}{Lemma}
\newtheorem{problem}{Problem}
\newtheorem{proposition}[lemma]{Proposition}
\newtheorem{theorem}[corollary]{Theorem}

\theoremstyle{remark}
\newtheorem{definition}{Definition}

\usepackage{template/macro}

\usepackage[normalem]{ulem}
\newcommand{\note}[1]{\textcolor{red}{#1}}
\newcommand*{\charles}[1]{\textcolor{blue}{[\textbf{CA:} #1]}}

\renewcommand{\note}[1]{}
\renewcommand*{\charles}[1]{}

\endlocaldefs

\begin{document}

\begin{frontmatter}

\title{Mode Estimation with Partial Feedback}
\runtitle{Mode Estimation with Partial Feedback}

\begin{aug}
\def\thefootnote{*}\footnotetext{These authors contributed equally to this work.}

\author[A]{\fnms{Charles}~\snm{Arnal\textsuperscript{*}}},
\author[B]{\fnms{Vivien}~\snm{Cabannes\textsuperscript{*}}}
\and
\author[C]{\fnms{Vianney}~\snm{Perchet}}

\address[A]{Datashape, Inria, Saclay, France}
\address[B]{Fundamental Artificial Intelligence Research (FAIR), Meta AI, New York, USA}
\address[C]{Center for Research in Economics and Statistics (CREST), ENSAE, Palaiseau, France}


\end{aug}

\begin{abstract}
    \charles{rendre les notes invisibles}
The combination of lightly supervised pre-training and online fine-tuning has played a key role in recent AI developments.
These new learning pipelines call for new theoretical frameworks.
In this paper, we formalize core aspects of weakly supervised and active learning with a simple problem: the estimation of the mode of a distribution using partial feedback.
We show\note{case} how entropy coding allows for optimal information acquisition from partial feedback, develop coarse sufficient statistics for mode identification, and adapt bandit algorithms to our new setting.
Finally, we combine those contributions into a statistically and computationally efficient solution to our problem.

\end{abstract}




\end{frontmatter}

\tableofcontents


\section{Introduction}

The mode of a distribution is a fundamental concept in statistics, serving as a key identifier for the most likely event to occur.
For instance, identifying modes of conditional distributions is the main task of classification algorithms --classification consists in learning the mapping $f^*(x) = \argmax_y p(y\,|\, x)$ for a joint distribution $p$ over inputs $x$ and classes $y$.
Traditionally, datasets were small enough to fully annotate samples before learning the modes of the underlying distributions.
However, with the increasing scale of machine learning problems, data collection has become a significant part of machine learning pipelines.
This is illustrated by the substantial efforts dedicated to data preprocessing to train foundational AI models \citep[see, e.g.,][]{Openai2023,Touvron2023}.
Moreover, it is foreseeable that future models will incorporate annotation feedback loops.
Indeed, fine-tuning foundational models already relies on various active learning strategies, such as reinforcement learning with human feedback \citep{Ziegler2020} and partial annotations \citep{Zhu2023}.
This makes theories of weakly-supervised, and active learning highly relevant to the machine learning community.
This paper introduces one of the simplest setups to combine active and weakly-supervised learning.
It focuses on partial annotations, and searches for the best algorithms to identify the mode of a distribution given a budget of annotations. 

\subsection{The Mode Estimation with Partial Feedback Problem}

The task at hand is the identification of the mode of a distribution $p\in\prob\cY$, where $\cY$ is a set of $m$ elements or classes, through partial feedback.
Let us denote as $\cY = \brace{y_1, \ldots, y_m}$ this set of classes, and assume that the $m$ elements $y_1, \ldots, y_m$ are indexed by decreasing probability, i.e. $p(y_i) \geq p(y_{i+1})$ for any $i < m$.
The goal is to find the most probable value $y_1$ of $\cY$, which we assume to be unique for simplicity, i.e.
\begin{equation}
    \label{eq:obj}
    y_1 = \argmax_{y\in\cY} \bP_{Y\sim p}(Y=y).
\end{equation}
To estimate the mode, we assume the existence of independent samples $(Y_j)$ distributed according to~$p$. 
However, the practitioner does not directly observe the samples; instead, they can sequentially acquire weak information on them.
Formally, at each time $t\in\bN$, the practitioner selects an index $j_t \in \bN$ and a set of potential labels $S_t \subset \cY$, and ask whether the sample $Y_{j_t}$ belongs to $S_t$ or not, leading to the observation of the binary variable
\begin{equation}
    \label{eq:question}
    \ind{Y_{j_t} \in S_t}.
\end{equation}
We call the process of checking whether $Y_{j_t} \in S_t$ a \textit{query}, leading to the following problem description:
\begin{problem}
    \label{Problem}
    Design a sequence of queries $(\ind{Y_{j_t} \in S_t})$ to efficiently estimate the mode of $p$, where each query $\ind{Y_{j_t} \in S_t}$ can depend on the previous observations of $(\ind{Y_{j_s} \in S_s})_{s<t}$.
\end{problem}
The overarching objective is to design and analyze efficient algorithms that output a good estimate of the mode using a minimal number of queries. 
Although real-world applications often come with peculiarities, Problem~\ref{Problem} is generic enough that its resolution should shed light on a large variety of online learning tasks with partial feedback.

Among possible variants of Problem~\ref{Problem}, one could forbid the experimenter from asking more than one query per sample, i.e. forcing $j_t = t$, add a contextual variable $X\in \cX$ that conditions the distribution of $\ind{Y\in S}$, assume that some queries are cheaper to make than others, or consider cases where random noise affects the observations $\ind{Y\in S}$.
Rather than the mode $y_1$ of the distribution $p$, we may want to identify instead one or all the classes $y$ such that $p(y) \geq p(y_1)-\epsilon$ for some $\epsilon >0$.
Finally, more structure could be added to the set $\cY$, rather than let it be a collection of unrelated classes without any meaningful interactions.
In particular, the queried sets $S$ could be restricted to belong to some predefined collection of sets $\cS$ --e.g., each class in $\cY$ represents an animal species, and you can ask whether the animal $Y$ is a feline, but not whether it belongs to $\{$cat, wolf, common snapping turtle$\}$.
More generally, the observation of $\ind{Y\in S}$ could be replaced by a random variable $F(Y,S)$ for some function $F$, so that $F(Y,S)$ need not be discrete or equal to $\sum_{y\in S}F(Y,S)$.

Problem~\ref{Problem} and its variations appear in many contexts.
A natural illustration would be a content-providing service, such as a social network app with a scrolling-centric interface, that tries to identify which type of content the user is most likely to like from a collection $\cY$ of possible topics. The app shows a batch of posts or videos to the user (which corresponds to the choice of a set $S$), and receives some measure of user satisfaction in return (such as their scrolling time); from this information, the app must design future batches of content and find the user's favourite topics.
Another natural example comes from advertising: a hotel can compose online ads using various combinations of pictures from a set $\cY$ of photos of a given room. Each combination corresponds to a set $S$, and the variable $\ind{Y\in S}$ is equal to $1$ when online visitor $Y$ clicks on the ad. The hotel tries to identify which pictures maximize the chances of the ad being clicked.
Similar situations can also arise in experimental sciences: a biologist tries to understand which genes $y$ in a collection $\cY$ of genes considered have the strongest effect on a certain property. To that end, they can test whether activating a set $S$ of genes results in the expression of the property.

\subsection{Related Work}

In terms of related problems, Problem \ref{Problem} bears resemblances to the active labeling framework introduced by \citet{Cabannes2022}.
This framework was motivated by dynamic pricing \citep{Cesabianchi2019}, active ranking \citep{Jamieson2011,Braverman2019}, and hierarchical classification \citep{Cesa-Bianchi2006,Gangaputra2006}. 
Additionally, a clear connection can be drawn with the well-studied task of best-arm identification in multi-armed bandit settings \citep{Bubeck2010},\footnote{%
    The connection is apparent if we force $i_t = t$ and $S_t$ to be a singleton for every $t\in\bN$.
} particularly in the context of combinatorial bandits \citep{Chen2016} and even more precisely of transductive linear bandits \citep{Fiez2019}.

In terms of techniques, our algorithms draw their inspirations from the entropic coding schemes of \citet{Huffman1952} and \citet{Vitter1987}, the elimination algorithm of \citet{Evendar2006}, together with the doubling trick of \citet{Auer1995}.
Finally, we notice that the expectation-maximum (EM) algorithm of \citet{Dempster1977} was notably motivated by the estimation of probability from partial observations, although we did not pursue  their approximate Bayesian approach.

Philosophically, our approach was motivated by the fact that Kolmogorov capacity, hence entropy coding, lies at the heart of statistical learning theory (see, e.g., \citet{Cucker2002}, as well as \citet{Grunwald2007}), suggesting that a contextual version of our setup might provide insightful perspectives on active learning.

\subsection{Performance Metrics}
\label{subsec:performance-metrics}

Various performance metrics can be applied to algorithms that output a guess $\hat y$ of the mode of $p$ after a certain number of queries.
We will focus on the probability of error, i.e., on the minimal $\delta$ such that
\begin{equation}
    \label{eq:zero-delta}
    \E[\ind{\hat y\neq y_1}] = \bP(\hat y \neq y_1) \leq \delta,
\end{equation}
as a function of the number of queries $T$. Here, the randomness is inherited from that of the samples $(Y_j)$ and of the algorithm.
Equation \eqref{eq:zero-delta} can be seen as the ``risk'' associated with the ``zero-one loss'', which is used in classification problems.
Other reasonable metrics include the expected difference $\E[p(y_1) - p(\hat y)]$, or the probability of $\epsilon$-approximate success $\bP(p(\hat y) \geq p(y_1) -\epsilon)$.
Despite those variations, the principles behind our algorithms can be readily adapted to different metrics.

Among additional nuances, in some contexts, a user might fix in advance a ``confidence'' level~$\delta$, and be interested in algorithms that minimize the (expected) number of queries $T$ to achieve it.
Reciprocally, one might fix a ``budget'' of queries $T$, and design algorithms that minimize the probability $\delta$ of false prediction using those queries.
These settings typically lead to equations of the shape
\begin{equation}
    \label{eq:alpha}
    \E[T] \leq \ln(C_1 / \delta) \alpha_1, \qquad\text{or respectively}\qquad \delta \leq C_2 \exp(-T / \alpha_2),
\end{equation}
where the constants $\alpha_i$ and $C_i$ depend on the unknown distribution $p$.
Once again, algorithm design principles are often agnostic to the settings differentiation, usually allowing the conversion of one type of bound into the other.
To offer quantitative comparisons, we focus on the joint asymptotic behavior with respect to $\delta$ and $T$, which is determined by the constants $\alpha$, removing $C$ from the picture.
This provides a simple metric: the smaller $\alpha$, the better the algorithm.

Note that our proofs also yield guarantees on the sample complexity of our algorithms; however, the expected number of queries required is a more relevant quantity, as it captures both the sample efficiency of an algorithm and the efficiency with which it extracts the necessary information from each sample, which is a key aspect of our setting. 

\subsection{Summary of Contributions}

\begin{table}[t]
    \centering
    \begin{tabular}{ccc}
        \begin{tabular}{|c|}\hline Exhaustive Search Algorithm~\ref{algo:exhaustive}        \\ \hline $\Delta_2^{-2} + \Delta_2^{-2}\log_2(m)$ \\ \hline\end{tabular}
          & & \\ 
        $\downarrow$ & & \\ 
        \begin{tabular}{|c|}\hline Adaptive Search Algorithm~\ref{algo:adaptive-gen}           \\ \hline $\Delta_2^{-2} + \Delta_2^{-2}\paren{\sum_{i\geq 1}p(y_i)\abs{\log_2 p(y_i)}}$ \\ \hline\end{tabular}
          & $\longrightarrow$ &  
        \begin{tabular}{|c|}\hline Elimination Algorithm~\ref{algo:elimination}                  \\ \hline $\Delta_2^{-2} + \sum_{i\geq 1}p(y_i)\Delta_i^{-2} \abs{\log_2 p(y_i)}$ \\ \hline\end{tabular}
          \\ 
        $\downarrow$ & & $\downarrow$ \\ 
        \begin{tabular}{|c|}\hline Truncated Search Algorithm~\ref{algo:truncated} \\ \hline $\Delta_2^{-2} + \Delta_2^{-2}\abs{\log_2 p(y_1)}$ \\ \hline\end{tabular}
          & $\longrightarrow$ &  
        \begin{tabular}{|c|}\hline Set Elimination Algorithm~\ref{algo:set-elimination}                     \\ \hline $\Delta_2^{-2} + \paren{\sum_{i\geq 1}p(y_i) \Delta_i^{-2}} \abs{\log_2 p(y_1)}$ \\ \hline\end{tabular}
          \\
    \end{tabular}
    \caption{
    Summary of the coefficient $\alpha$ as per~\eqref{eq:alpha} found for different methods, up to universal constants when $p(y_1)$ is bounded away from $1$, and with $\Delta_1 = \Delta_2$.
    The arrows indicate improvements.
    }
    \label{tab:summary}
\end{table}

Our first contribution is to {\em introduce a new problem, Problem~\ref{Problem}, which captures key aspects of most online learning tasks with partial feedback.}
It is simple enough to allow for rigorous theory to be developed, yet naturally generalizes to match practical scenarios from the real world.
We introduce several important ideas to efficiently solve Problem~\ref{Problem}, which lead to increasingly refined algorithms, whose performances can be compared in terms of the coefficient $\alpha$ introduced in \eqref{eq:alpha}.
An outline is provided in Table~\ref{tab:summary}, up to universal multiplicative constants.
Though each algorithm is discussed in detail later in the article, let us introduce them here.

The most naive one, the {\em Exhaustive Search} Algorithm~\ref{algo:exhaustive}, consists in fully identifying each sample $Y_j$ through binary search. At any time $t$, it outputs the most frequent class among the identified samples as an estimation of the true mode.
For a tolerated probability of error $\delta\in (0,1]$ and a class $y_i \in \cY$, a number of samples proportional to $\Delta_{i}^{-2}\ln(1/\delta)$ is needed to correctly identify $y_1$ as the mode among $\{y_1,y_i\}$, where
 \begin{equation}
     \label{eq:delta-*}
     \Delta_{i}^2 :=  -\ln\parend[\Big]{1 - (\sqrt{p(y_1)}-\sqrt{p(y_i)})^2}.
 \end{equation}
Unsurprisingly, $\Delta_i^{-2}$ is increasing with $p(y_i) \in [0, p(y_1))$ and grows infinitely large when $p(y_i)$ get closer to $p(y_1)$.
The probability of error of Algorithm~\ref{algo:exhaustive} is dominated by the probability of mistakenly picking the second most likely candidate $y_2$ as the mode, leading to the asymptotic performance $\alpha = \Delta_2^{-2} \ceil{\log_2(m)}$, as detailed in Section~\ref{sec:exhaustive}.
It can be improved by identifying new samples with an entropy-based dichotomic search that uses a learned empirical distribution $\hat p$ on $\cY$, which yields Algorithm~\ref{algo:adaptive}. Asymptotically, it replaces the $\log_2(m)$  queries per sample of Algorithm~\ref{algo:exhaustive} by an expected number of queries equal to the entropy $H(p):=\sum_{i\geq 1}p(y_i)\abs{\log_2(p(y_i))} \leq \log_2(m)$ of $p$, plus one query per sample due to some boundary effects, as reflected in Table~\ref{tab:summary} and as explained in Subsection~\ref{subsec:entropy_search}.
Showing how {\em online learning with partial feedback benefits from entropy coding} is our second contribution.

There are two disjoint ways to further improve Algorithm~\ref{algo:adaptive}.
The first one exploits the main characteristic of our problem: the possibility of asking for partial information on samples at a lower cost than for complete identification.
In particular, when searching for the mode by identifying samples with entropy-based techniques using Huffman trees, one can stop the identification procedure at roughly the depth of the mode in the tree, as classes that are deeper in the tree are unlikely candidates.
This idea leads to the design of the {\em Truncated Search} Algorithm~\ref{algo:truncated}, detailed in Section~\ref{sec:truncated}.
It reduces the asymptotic number of queries per sample from a constant plus the average depth of leaves in a Huffman tree $H(p)$ to a constant plus the minimal depth $\abs{\log_2(p(y_1))}$, as seen in Table~\ref{tab:summary}.
The second amelioration comes from the adaptation of bandit algorithms to our setting.
We develop in Section~\ref{sec:elim} the {\em Elimination} Algorithm~\ref{algo:elimination}, that discards mode candidates as soon as they seem unlikely to be the true mode.
The class  $y_i \in \cY$ can be eliminated after roughly $\Delta_i^{-2}\ln(1/\delta)$ samples, leading to an asymptotic performance $\alpha \simeq \Delta_2^{-2} + \sum_{i\geq 1}\Delta_i^{-2}\, p(y_i)\abs{\log_2(p(y_i))}$, with the abuse of notation $\Delta_1^{-2} := \Delta_2^{-2}$.

While Algorithms~\ref{algo:truncated} and~\ref{algo:elimination} are both improvements over the Adaptive Exhaustive Search Algorithm~\ref{algo:adaptive}, neither is strictly better than the other.
When $p(y_1)-p(y_2)$ is very small compared to the other gaps  $p(y_1)-p(y_i)$, the advantage goes to the Elimination Algorithm,\footnote{%
    Consider the distribution $p(y_1) = 2/m$, $p(y_2) = 2/m-1/m^2$ and $p(y) = (1-p(y_1)-p(y_2))/(m-2)$ for all other $y\in\cY = \{y_1,\ldots,y_m\}$. Then $\alpha_E = \sum_{i\geq 1}\Delta_{i, *}^{-2}\, p(y_i)\abs{\log_2(p(y_i))} = O(\alpha_{TS}/m)$ as $m\rightarrow\infty$.
}
while it goes to the Truncated Search Algorithm when there are many $y\in \cY$ with small mass $p(y)\ll p(y_1)$ and $p(y_1)-p(y_2)$ is not too small.\footnote{%
    Consider the distribution $p(y_1) = 1/2$ and $p(y) = 1/(2(m-1))$ for all other $y\in\cY = \{y_1,\ldots,y_m\}$. Then $\alpha_E \rightarrow \infty$ while $\alpha_{TS} = \sum_{i\geq 1}\Delta_2^{-2}\, p(y_i)\abs{\log_2(p(y_1))} = O(1)$ as $m\rightarrow\infty$.
}
We combine the ideas of each algorithm and get the best of both worlds with the {\em Set Elimination} Algorithm~\ref{algo:set-elimination}, presented in Section~\ref{sec:set-elim}, which can be understood either as grouping together low mass classes before applying an elimination procedure to the resulting partitions, or as refining the truncated search procedure by taking into account confidence intervals.
This is reflected in the corresponding coefficient $\alpha = \Delta_2^{-2} + \sum_{i \geq 1} \Delta_i^{-2}\, p(y_i)\abs{\log_2(p(y_1))}$,  where the number of samples $\Delta_i^{-2}$ is as for the Elimination Algorithm, while the expected number of queries needed for each sample $ 1 +\abs{\log_2(p(y_1))}$ comes from the Truncated Search Algorithm, resulting in the best asymptotic performance among the proposed algorithms. 
\emph{This sophisticated solution to Problem~\ref{Problem}, together with its implementation available online, is our final contribution}.

In conclusion, our main contributions are summarized as follows.
\begin{enumerate}
    \item Introducing the problem of mode estimation with partial feedback, Problem~\ref{Problem}.
    \item Unveiling links between adaptive entropy coding and active learning.
    \item Combining adaptive entropy coding, coarse search procedures and bandits-inspired principles into Algorithm~\ref{algo:set-elimination}, an intuitive yet efficient solution to Problem~\ref{Problem}.
\end{enumerate}
Last but not least, we provide a code base to help researchers advance our knowledge of weakly supervised online learning, available at \url{www.github.com/VivienCabannes/mepf}.

\section{The Empirical Mode Estimator}
\label{sec:mode}

In the first part of this article, we will design mode estimation algorithms that {\em fully identify each sample} $Y_j$ one after the other.
Given the identification of $n$ samples, those algorithms {\em estimate the mode of $p$ as the empirical mode among the $n$ samples,}
\begin{equation}
    \label{eq:mode-emp}
    \hat y_n := \argmax_{y\in\cY} \sum_{j\in[n]} \ind{Y_j = y},
\end{equation}
with ties broken arbitrarily.
A tight characterization of the performance of the estimator~\eqref{eq:mode-emp} is offered by the following theorem.

\begin{theorem}[Empirical mode performance]
    \label{thm:empirical-mode}
    Let $(Y_j)_{j\in[n]}$ be $n$ independent variables sampled according to $p\in\prob{\cY}$. Then the probability of error of~\eqref{eq:mode-emp} satisfies
    \begin{equation}
        \label{eq:Sanov}
        \exp\paren{-n\Delta_2^2 - m\ln(n+1) - c_p} \leq \bP(\hat y_n \neq y_1) \leq \exp\paren{-n\Delta_2^2}.
    \end{equation}
    where $\Delta_2$ is defined in Eq. \eqref{eq:delta-*}, and $c_p$ is some constant that depends on $p$.
\end{theorem}

When accessing $n$ samples, the performance of empirical mode $\hat y_n$ cannot be bested without additional information on $p$, which leads to the following corollary.

\begin{corollary}[Minimax lower bound]
    \label{cor:minimax}
    For any distribution $p_0\in\prob\cY$, and any algorithm $\cA$ that predicts $\hat y:=\cA((Y_j)_{j\in[n]})$ based on $n$ observations $(Y_j)$, there exists a permutation $\sigma\in\fS_m$ such that, when the data are generated according to $p\in\prob\cY$ defined through the formula $p(y) = p_0(\sigma(y))$, the lower bound~\eqref{eq:Sanov} holds for this algorithm. 
\end{corollary}

As one needs at least a single query per sample to gain any meaningful information, Corollary \ref{cor:minimax} states that the number of queries $T$ need to be greater than $\Delta_{2}^{-2} \ln(1/\delta)$, up to higher order terms, to reach precision $\delta$ as defined by Equation \eqref{eq:zero-delta}.
The main challenge is thus to get as close to this lax lower bound as possible.

\subsection{Proof of Theorem \ref{thm:empirical-mode}}

In this subsection, we prove Theorem \ref{thm:empirical-mode}.
We follow a proof of Sanov's theorem due to \citet{Csiszar2011}, together with an explicit computation of an information projection.
Let us partition all possible sequences of observations $(Y_j)_{j\in[n]}$ according to their empirical distribution defined as 
\[
    \hat p_{(Y_j)}(y) := n^{-1}\sum_{j\in[n]}\ind{Y_j = y}.
\]
To do so, we define for any such empirical distribution $q \in \prob\cY\cap n^{-1}\cdot\bN^\cY$ the type class
\[
    \cT(q) = \brace{(y_t) \in \cY^n\midvert \forall\,y\in\cY;\quad \hat p_{(y_t)}(y) = q(y)}.
\]
The event $\{\hat y_n \neq y_1\}$ is the union over all the values that $\hat p$ can take of the events ``$\hat p$ does not have the right mode'', which we can write as a union of disjoint events according to the type of $(Y_j)$, which leads to the bound 
\[
    \sum_{q\in \cQ_{n,-}} \Pbb_{(Y_j)\sim p}\parend[\big]{(Y_j) \in \cT(q)} \leq
    \Pbb_{(Y_j)_{j\in[n]}}(\hat y \neq y_1) \leq \sum_{q\in \cQ_{n,+}} \Pbb_{(Y_j)\sim p}\parend[\big]{(Y_j) \in \cT(q)},
\]
where  we account for the cases where $\hat p$ has several modes by differentiating
\begin{equation}
    \label{eq:Qn-}
    \cQ_{n,-} = \brace{q \in \prob\cY \cap n^{-1}\cdot\bN^\cY \midvert y_1 \notin \argmax q(y)},
\end{equation}
and
\begin{equation}
    \label{eq:Qn+}    
    \cQ_{n,+} = \brace{q \in \prob\cY \cap n^{-1}\cdot\bN^\cY \midvert \argmax q(y) \neq \argmax p(y)}.
\end{equation}

We would like to enumerate the different classes in the previous sums, as well as their probability.
A few lines of derivations lead to the equality,
\[
    \bP((Y_j) = (z_j)) = 2^{- n(H(\hat p_{(z_j)}) + D(\hat p_{(z_j)}\| p)) }.
\]
Here, $D$ is the Kullback-Leibler divergence, and $H$ the entropy
\[
    D(q\| p) = \E_{Y\sim q}[\log_2(\frac{q(Y)}{p(Y)})], \qquad
    H(q) = \E_{Y\sim q}[-\log_2(q(Y))].
\]
As a consequence, using the exchangeability of the $(Y_j)$,
\[
    \bP((Y_j) \in \cT(q)) = 
    \bP(\hat p_{(Y_j)} = q) = \card{\cT(q)} 2^{- n(H(q) + D(q\| p)) }.
\]

We are left with the computation of the cardinality of each $\cT(q)$.
This is nothing but 
\[
    \card{\cT(q)} = \binom{n}{(nq(y))_{y\in\cY}} = \frac{n!}{\prod_{y\in\cY} (nq(y))!}.
\]
This cardinality can be bounded with probabilistic arguments, as shown in Theorem 11.1.3 (which is 12.1.3 in the 2nd edition) of \citet{Cover1991},\footnote{%
    Slightly tighter bounds can be derived from the fact that, as proven by \citet{Robbins1955},
    \[
        k! = \sqrt{2\pi} k^{k+1/2}e^{-k+r_k}, 
       \quad\text{with}\quad 0 \leq \frac{1}{12k + 1} \leq r_k \leq \frac{1}{12k} \leq 1.
    \]
}
\[
     (n+1)^{-m} 2^{nH(q)} \leq \card{\cT(q)} \leq 2^{nH(q)}.
\]
Collecting the different pieces so far, we reach the conclusion,
\[
    (n+1)^{-m} \sum_{q\in \cQ_{n,-}} 2^{-nD(q \| p)} \leq \Pbb(\hat y \neq y_1) \leq \sum_{q\in \cQ_{n,+}} 2^{-nD(q \| p)}.
\]
From there, using the fact that the cardinality of $\cQ_{n,+}$ is at most $(n+1)^m$, we get the rough bound
\begin{equation}\label{eq:first_bound_Sanov}
   (n+1)^{-m}  \max_{q\in \cQ_{n,-}} 2^{-nD(q \| p)} \leq \Pbb(\hat y \neq y_1) \leq (n+1)^m \max_{q\in\cQ_{n, +}}2^{-nD(q \| p)}.
\end{equation}
It is useful to define the ``limit'' $\cQ = \brace{q \in \prob\cY \midvert \argmax q(y) \neq \argmax p(y)}$
of the sets $\cQ_{n,-}$ and $\cQ_{n,+}$.
We are left with the computation of the so-called ``information projections'' $\min_{q\in \cQ_{n,-}} D(q \| p)$ and $\min_{q\in \cQ_{n,+}} D(q \| p)$.

\begin{restatable}{lemma}{infoprojection}
    \label{lem:kl-proj-short}
    For any distribution $p\in\prob{\cY}$, there exists a constant $c_p$ such that for any $n\in\bN$ with $\cQ_{n,-}$ and $\cQ_{n,+}$ defined by Equations \eqref{eq:Qn-} and \eqref{eq:Qn+},
    \[
        \frac{\Delta_2^2}{\ln(2)} + \frac{c_p}{n\ln(2)} \geq
		\min_{q\in \cQ_{n,-}} D(q \| p) \geq
        \min_{q\in \cQ_{n, +}} D(q \| p)\geq \min_{q\in \cQ} D(q \| p) = \frac{\Delta_2^2}{\ln(2)}.
    \]
\end{restatable}
This lemma is proved in Appendix \ref{app:mode}.
Combining it with Equation \eqref{eq:first_bound_Sanov}, we find that
\[
    \exp\left( -n \Delta_2^2 - c_p -  m\ln(n+1)\right)
    \leq \Pbb(\hat y \neq y_1) \leq \exp\left( -n \Delta_2^2 + m\ln(n+1)\right).
\]
Finally, we note that an argument of \citet{Dinwoodie1992} shows that $\ln(\bP(\hat y_n \neq y) )/ n$ is always below its limit proved by \citet{Sanov1957}, which yields the stronger bound 
\[\Pbb(\hat y \neq y_1) \leq \exp( -n \Delta_2^2)\]
which we reported in Theorem \ref{thm:empirical-mode}.

\subsection{Minimax Optimality}

This subsection proves Corollary \ref{cor:minimax}.
To turn Theorem \ref{thm:empirical-mode} into the minimax Corollary \ref{cor:minimax}, we can use a prior on $p$ such that the optimal Bayesian algorithm is the mode estimation algorithm.
To do so, let us endow the simplex $\prob\cY$ with a distribution $D$ and try to minimize
\[
    \cE(\cA) = \E_{p\sim D}\bracket{\E_{(Y_j)}\bracket{\ind{\cA((Y_j)) \neq y^*_p} \midvert p}},
\]
where $\cA((Y_j))$ is the mode predicted by an algorithm $\cA$ upon observing the independent samples $(Y_j)$ generated according to $p\in\prob\cY$, and $y^*_p$ is the mode of $p$.
To find the optimal algorithm $\cA^*$, we can invert the expectation as
\[
    \cE(\cA) = \E_{(Y_j)}\bracket{\E_{p\sim D}\bracket{\ind{\cA((Y_j)) \neq y^*_p} \midvert (Y_j)}}.
\]
This leads to the optimal algorithm
\[
    \cA^*((Y_j)) = \argmin_{y\in\cY} \E_p\bracket{\ind{y \neq y^*_p} \midvert (Y_j)}
    = \argmax_{y\in\cY} \bP_{p\sim D}\paren{y^*_p=y \midvert (Y_j)}.
\]
It follows from $\cE(\cA) \geq \cE(\cA^*)$ that there exists at least one distribution $p$ in the support of $D$ such that the error made by any algorithm $\cA$ on this distribution cannot be better that the one made for $\cA^*$, i.e.,
\[
    \E_{(Y_j)}\bracket{\ind{\cA((Y_j)) \neq y^*_p} \midvert p}
    \geq \E_{(Y_j)}\bracket{\ind{\cA^*((Y_j)) \neq y^*_p} \midvert p}.
\]
One can define a prior such that $\cA^*$ corresponds to the empirical mode, i.e., $\cA^*((Y_j)) = \hat y_n$.
In particular, Corollary~\ref{cor:minimax} follows from taking $D$ as the uniform distribution over the set of permutations of the given distribution $p_0$, i.e., $\fS_m\cdot p_0 = \brace{p\midvert \exists\,\sigma\in\fS_m, p = \sigma_\# p_0}$.
Since any algorithm has to have its expected performance over the distribution $p$ bounded by the one of $\cA^*$, for any algorithm, there exists at least one distribution $p \in\fS_m\cdot p_0 $ such that its performance is no better than the one of $\cA^*$.

\subsection*{Aside on User-Knowable Bounds}
For a user that does not know $\Delta_y$, the bound of Theorem \ref{thm:empirical-mode} is not practically helpful, and the practitioner might be interested in stronger formal guarantees.
In particular, \citet{Valiant1984} introduced the notion of $(1-\delta)$-probably $\epsilon$-approximately correct estimator, for some $\epsilon, \delta > 0$, which reads $\bP(p(\hat y) < p(y_1) - \epsilon) \leq \delta$.
Proofs based on concentration inequalities can usually be reworked to derive such $(\epsilon, \delta)$-PAC bounds by replacing some quantity $\Delta$, related to differences $p(y) - p(y_1)$, by $\epsilon$ in the bound on the error $\delta$.
This is notably the case for Theorem \ref{thm:empirical-mode}.
One can also estimate $\hat\Delta_y$ and use plug-in techniques.
The second half of this paper will provide algorithms, namely the Elimination and Set Elimination Algorithms, such that the user knows, and furthermore can choose, the probability $\delta$ with which the algorithm will fail to output the true mode --in the bandit literature, such algorithms are called $(0, \delta)$-PAC.

\section{Exhaustive Dichotomic Search Procedures}
\label{sec:exhaustive}

This section introduces baseline algorithms based on exhaustive search procedures.
To leverage the empirical mode estimator, a naive baseline consists in fully identifying each sample $Y_j$ one after the other.
Using a  binary search, this can be done with $\ceil{\log_2(m)}$ queries \eqref{eq:question} per sample.
We can thus fully identify $n$ samples with $T = n \ceil{\log_2(m)}$ queries.
Together with Theorem \ref{thm:empirical-mode}, we deduce that this naive baseline gets to precision $\delta$ as defined by Equation \eqref{eq:zero-delta} with a number of queries bounded by $T = \Delta_{2}^{-2}\ceil{\log_2(m)}\ln(1/\delta)$.
Our first proposed improvement over this baseline is to refine binary searches through entropy codes in order to {\em minimize the average number of binary questions asked to fully identify each sample} $Y_j$, reducing the number of queries per sample from $\ceil{\log_2(m)}$ to a quantity close to the entropy $H(p)$ of $p$ \citep{Shannon1948}.

\subsection{Entropy Coding}
\label{subsec:coding}

In order to identify each sample $Y_j$ with the least amount of queries, we will use binary search algorithms stemming from coding theory.
This section fixes terminology and adds self-contained details on the matter.

\begin{definition}[Binary Tree {\em etc.}]
    For a set of elements $\cY$, we define leaves as abstract objects $V_y$ for all $y\in \cY$.
    From leaves, we define nodes $V$ as abstract objects associated with a right child $V_1 = r(V)$ and a left child $V_2 = l(V)$ which are either nodes or leaves.
    A vertex is an abstract object $V$ which is either a node or a leaf.
    A vertex $V_1$ is a descendant of $V_2$, denoted by $V_1\lhd V_2$, if it can be built from the composition  $V_1 = s_1\circ s_2 \circ s_3 \circ \cdots \circ s_d(V_2)$ for $d\in\bN$ and $s_i \in \brace{r, l}$
    The descendants of a node $V$ are all the vertices that can be built from the composition $V' \lhd V$.
    A binary tree $\cT = \cT(R)$ is defined from a root node $R$ with a finite number of descendants forming a collection of vertices, such that each $V$ in this collection of vertices is defined by a unique path $V = s_1\circ s_2 \circ s_3 \circ \cdots \circ s_d(R)$, with $d = D_\cT(V)\in\bN$ being known as the depth of $V$ in $\cT$.
\end{definition}

A binary tree is associated with a prefix code.

\begin{definition}[Vertex code]
    \label{def:code}
    The code $c_\cT(V) \in \brace{0, 1}^{D_\cT(V)}$ of vertex $V$ in $\cT$ is defined as the unique path to go from the root node of $\cT$ to $V$ by reading $c_\cT(V)$ and recursively advancing to the right child if reading a $1$ and to the left child if reading a $0$.
\end{definition}

The concept of binary tree is graphically intuitive.
Let us provide an example of a tree, and annotate each node in the tree with its code as per Definition \ref{def:code}.
In the following illustration, each left (resp. right) child of a node are represented below the node on the left (resp. on the right).
We emphasize the separation between right and left with a vertical bar.
For example, if $V = r \circ l\circ l\circ r(R)$, then $c_\cT(V) = 1001$.

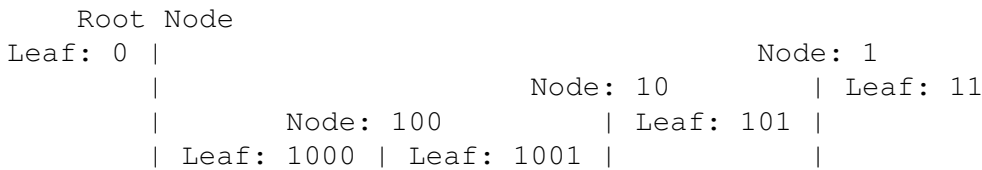
\begin{verbatim}
    Root Node
Leaf: 0 |                                  Node: 1
        |                     Node: 10        | Leaf: 11
        |       Node: 100         | Leaf: 101 |
        | Leaf: 1000 | Leaf: 1001 |           |
\end{verbatim}

Prefix codes have been heavily studied in information theory, where the goal is to transform a set $\cY$ in order to describe its elements $y$ with a sequence of bits $c(y)$.
In particular, the minimal number of bits transmitted when encoding sequences of tokens $(Y_j)\in\cY^n$ into the concatenation of the codes $(c(Y_j))_{j\in[n]}$ is linked with the entropy of the empirical distribution of the $y \in \cY$ in $(Y_j)$.
There is a well known fundamental limit, for any $p\in\cY$,
\[
    \min_{\cT}\E_{Y\sim p}[D_\cT(Y)] \geq H(p):= \E_{Y\sim p}[-\log_2 p(Y)],
\]
where $D_\cT(y)$ is seen as the length of the code $c(y)$, the expectation is seen as the average code length, and the minimization as the search for the best coding algorithm to minimize message length.

A simple algorithm to obtain optimal codes was found by \citet{Huffman1952} --we reproduce it in Algorithm \ref{algo:huffman}.
It takes as input a set of elements $ \cY$ such that some positive value $v(y)$ is associated to each $y\in \cY$.
It outputs a tree whose leaves are the elements of $\cY$ and whose vertices also come associated to values, and which satisfies certain conditions with respect to those values.
When the value $v(y)$ of element $y\in \cY$ is equal to $p(y)$ for some distribution $p\in \Pbb(\cY)$, which is assumed in Algorithm~\ref{algo:huffman}, then the prefix code of the resulting tree on the elements of $\cY$ is optimal with respect to the distribution $p$, i.e. it minimizes $ \min_{\cT}\E_{Y\sim p}[D_\cT(Y)]$ \citep[see][]{Huffman1952}.

To describe Algorithm \ref{algo:huffman} succinctly, we introduce two notions of vertex orderings.
The first comes from the value $v(y)$ associated by hypothesis to each $y\in \cY$, which would typically be $p(y)$, $N(y)$ or $N(y) / n$, respectively the probability of $y$ for some distribution $p$, the empirical occurrence counts of $y$ for some set of samples, or its empirical probability.
Given a tree whose leaves $V_y$ are in bijection with the elements of $\cY$, let us associate to $V_y$ the value of the corresponding $y\in \cY$, i.e. $v(V_y) = v(y)$.
The value of a node is then defined recursively as the sum of its children value $v(V) = v(r(V)) + v(l(V))$.
This leads to the following notion of ordering.
\begin{definition}[Node ordering]
    \label{def:node-ordering}
    Two vertices $V_1$ and $V_2$ satisfy the partial ordering $V_1 \vdash V_2$ if $v(V_1) < v(V_2)$ or if $v(V_1) = v(V_2)$ and $V_2$ is a node while $V_1$ is a leaf.
    The two vertices are equivalent, which we write $V_1 \sim_\vdash V_2$, if $V_1 \not\vdash V_2$ and $V_2 \not\vdash V_1$.
\end{definition}
The second ordering orders the nodes from bottom to top, and left to right.
\begin{definition}[Code ordering]
    \label{def:code-ordering}
    Two vertices $V_1$ and $V_2$ in a tree $\cT$ satisfy the total ordering $V_1 <_\cT V_2$ if $D_\cT(V_1) > D_\cT(V_2)$ or if $D_\cT(V_1) = D_\cT(V_2)$ and $c_\cT(V_1) < c_\cT(V_2)$ with respect to the lexicographical order.
\end{definition}

Once again, this is a highly visual concept. Let us number the vertices of the previous tree according to the total ordering from Definition \ref{def:code-ordering}.

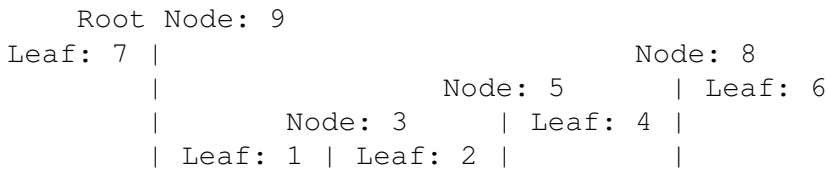
\begin{verbatim}
    Root Node: 9
Leaf: 7 |                           Node: 8
        |                Node: 5      | Leaf: 6
        |       Node: 3     | Leaf: 4 |
        | Leaf: 1 | Leaf: 2 |         |
\end{verbatim}

\begin{algorithm}
\caption{Huffman Scheme}
\KwData{Set of elements $\cY$ endowed with a probability distribution $p\in\prob\cY$.}
Create vertices $V_y$ for each $y\in\cY$ with value $v(V_y) = p(y)$;\\
Sort all the node into a heap $\cS$ according to the comparison $\vdash$ (Definition \ref{def:node-ordering});\\
\While{$\cS$ has more than one element}{
    Pop $V_1$, $V_2$ the respective smallest elements in $\cS$;\\
    Merge them into a parent node $V$ with $V_1$ as the left child and $V_2$ the right one;\\
    Insert $V$ into the heap $\cS$ with its value $v(V) = v(V_1) + v(V_2)$;
}
Set the remaining node $V$ in the heap $\cS$ as the root node of $\cT = \cT(V)$;\\
\KwResult{Huffman tree $\cT$}
\label{algo:huffman}
\end{algorithm}

Let us illustrate Algorithm \ref{algo:huffman} with the count vector $N = (69, 14, 8, 6, 3)$.
At the first iteration, we set all nodes in a heap.
Let us represent this heap as a sorted list with comma-separated elements.
\begin{verbatim}
Leaf 3; Leaf 6; Leaf 8; Leaf 14; Leaf 69;
\end{verbatim}
Then we merge the two smallest ones and add the result in the heap.
Let us picture descendants below the nodes.
\begin{verbatim}
Leaf 8; Node 9; Leaf 14; Leaf 69;
   Leaf 3 | Leaf 6
\end{verbatim}
We iterate the process.
\begin{verbatim}
Leaf 14; Node 17; Leaf 69 
    Leaf 8 | Node 9
        Leaf 3 | Leaf 6
\end{verbatim}
Now the heap is only made of a node with value $31$ and a leaf with value $69$.
\begin{verbatim}
    Node 31; Leaf 69;
Leaf 14 | Node 17
     Leaf 8 | Node 9
         Leaf 3 | Leaf 6
\end{verbatim}
We end up with a binary tree.
\begin{verbatim}
                                     Node: 100       
       Node: 31                        | Leaf: 69
Leaf: 14 |       Node: 17              |           
         | Leaf: 8 |       Node: 9     |           
         |         | Leaf: 3 | Leaf: 6 |             
\end{verbatim}

In this paper, we introduce algorithms that build and adapt Huffman trees with respect to a distribution $p$ that is simultaneously being learnt online. This requires updating the trees on the fly.
We can do so with the rebalancing algorithm of \citet{Vitter1987}, which we present in Algorithm~\ref{algo:Vitter}.
This algorithm uses integer values (which corresponds to empirical counts), and updates the tree in reaction to an increase in value of $+1$ for a single leaf.
Of course, it can be repeatedly applied to account for greater changes.

\begin{algorithm}[ht]
\caption{Vitter Rebalancing}
\KwData{A Huffman tree $\cT$ with counts $v(V) \in \bN$. A node $V$.}
    Swap $V$ with the biggest $V_1$ in the sense of $<_\cT$ such that $V \sim_\vdash V_1$;\\
    Update $v(V) := v(V) + 1$;\\
    Swap $V$ with the smallest $V_2$ in the sense of $<_\cT$ such that $V \vdash V_2$;\\
    \uIf{$v(V) = v(V_2)$}{
        Update the new parent of $V_2$ by calling this algorithm on $(\cT, V_2)$;
    }
    \Else{
        Update the new parent of $V$ by calling this algorithm on $(\cT, V)$;
    }
\KwResult{Updated Huffman tree $\cT$.}
\label{algo:Vitter}
\end{algorithm}

One can prove the following properties of the Huffman tree construction and rebalancing.

\begin{proposition}[\citet{Vitter1987}]
    When $\cT$ is built with Algorithm \ref{algo:huffman} with some initial value $v(V_y) = \sum_{j\in[t_0]} \ind{Y_j=y} > 0$ (for some set of samples $(Y_j)_j \subset \cY$) and is updated incrementally at each timestep $t\in\bN$ with respect to Algorithm \ref{algo:Vitter} and the observation of new samples $Y_t \in \cY$, the total ordering $<_\cT$ (Definition \ref{def:code-ordering}) is always compatible with the partial ordering $\vdash$ (Definition \ref{def:node-ordering}).
\end{proposition}

Note that some subtleties are to be taken into account at initialization to deal with leaves that have not yet been observed.
In coherence with the construction of \citet{Vitter1987}, we consider a special ``not yet observed'' node which defines a balanced subtree containing all the unobserved leaves as its descendants.
When observing a new element $y$, we remove it from this subtree, re-balance the subtree, and create a new node whose left child is the ``not yet observed'' node and whose right one is the leaf corresponding to this $y$. This new node is set at the former place of the ``not yet observed'' node.

Finally, the following code property will be useful to derive statistical guarantees for our algorithms.

\begin{definition}
  	\label{def:balanced}
    A coding scheme $\cA$ that associates a binary tree $\cT_\cA(p)$ to a probability vector $p\in\prob\cY$ is said to be $C$-balanced if $|c_{\cT_\cA(p)}(y)| \leq C\ceil{\log_2(p(y))}$ for all $y\in \cY$.
\end{definition}

Shannon codes are built explicitly to be $1$-balanced \citep{Shannon1948}.
However, in contrast with Huffman coding, Shannon coding does not provide the lowest expected.
Huffman codes are not always $1$-balanced, but the following lemma, which we prove in Appendix~\ref{app:exhaustive}, states that they are at least $2$-balanced --in fact, one can prove a slightly better constant than 2.

\begin{restatable}{lemma}{huffmandepth}
    \label{lem:depth_in_Huffman_tree}
    Let $\cT$ be a Huffman tree with respect to a value function $v$ on its vertices such that $v(R) =1$, where $R$ is the root of $\cT$. Then for any vertex $V$ of $\cT$, we have
    \[
        D_\cT(V) \leq 2 \ceil{\log_2(1/v(V))}.
    \]
    In other words, Huffman codes are $2$-balanced.
\end{restatable}

\subsection{Exhaustive Search with Fixed Coding}
As mentioned at the start of the section, our first, and rather naive, proposed method of solving Problem~\ref{Problem} is to fully identify each sample  $Y_i$ using a fixed search procedure, 
i.e., a set of $q$ predefined questions $(\ind{y\in S_i})_{i\in[q]}$ such that any element $y\in\cY$ is fully identified by the values of the functions $\ind{y\in S_i}$.
Such a search procedure an be mapped to a prefix code $c:\cY\to\brace{0,1}^q$ that associates any $y$ to the binary code $c(y) = (\ind{y\in S_i})_{i\in[q]}$. It can also be mapped to a binary tree, the code describing branching properties to reach the leaf $y$ from the root node.
Each question can be seen as eliciting one bit in the code of $Y_j$, or equivalently going down one node in the corresponding tree.
Reciprocally, a code $(c(y)_j)_{j\in[q]}$ is associated with the search procedure considering $S_j = \brace{y\in\cY\midvert c(y)_j = 1}$.
In simple terms, $S_j$ enumerates all elements whose code $c$ contains a $1$ in the $j$-th position.
This strategy for solving Problem~\ref{Problem} is formalized in Algorithm~\ref{algo:exhaustive}.
The average number of questions asked by such a search procedure reads $\sum_{y\in\cY} p(y)|c(y)|$ where $|c(y)|$ is the length of the code of $y$, or equivalently the depth of $y$ in the associated binary tree.
Consequently, if the predefined search procedure is a simple binary search, the length of the code of any element is at most $\ceil{\log_2(m)}$, and Theorem~\ref{thm:empirical-mode} guarantees that we need at most $T = \ln(1/\delta) \alpha$ queries with $\alpha  =\Delta_2^{-2}  \ceil{\log_2(m)} $ to bound the probability of error by $\delta$.

\begin{algorithm}
\caption{Fixed Coding Exhaustive Search}
\KwData{Set of classes $\cY=\brace{y_1, \ldots, y_m}$ endowed with $p\in\prob\cY$. A code $c:\cY\to\brace{0, 1}^d$ (by default, we let $d = \ceil{\log_2(m)}$ and $c(y_i)$ be the binary representation of the number $i$ with zeros in front)}
    \For{$j\in[n]$}{
        Get new sample $Y_j\sim p$ and query $c_k(Y_j)$ for $k\in\{1,\ldots, \textrm{length}(c(Y_j))\}$ until $Y_j$ is identified;
}
Set $\hat y = \argmax_{y\in\cY} N(y)$, where $N(y) = \sum_{j\in[n]} \ind{Y_j=y}$;\\
\KwResult{Estimated mode $\hat y = \argmax \sum_{j\in[n]} \ind{Y_j=y}$ of $p$.}
\label{algo:exhaustive}
\end{algorithm}

\subsection{Exhaustive Search with Adaptive Coding}
\label{subsec:entropy_search}

Ideally, we would like to apply Algorithm~\ref{algo:exhaustive} with a code $c$ that is optimal with respect to the distribution $p$.
As $p$ is not known {\em a priori}, we need to learn the code on the fly.
This leads to Algorithm \ref{algo:adaptive-gen}, where a code is adapted iteratively to best reflect the current estimate $\hat p$ of $p$ based on past observations.
As shown by the following theorem, the need to learn a code online does not impact too much the complexity of Algorithm \ref{algo:adaptive} in comparison to that of Algorithm \ref{algo:exhaustive} with an optimal code.

\begin{algorithm}
\caption{Adaptive Coding Exhaustive Search}
\KwData{Set of classes $\cY=\brace{y_1, \ldots, y_m}$ endowed with $p\in\prob\cY$.}
    Initialize $N(y) = 0$ for all $y\in \cY$; A coding scheme $\cA$, an initial tree $\cT$;\\
    \For{$j\in[n]$}{
        Get new sample $Y_j\sim p$ and identify it by querying the entries of $c_\cT$;\\
        Update $N(Y_j) = N(Y_j) + 1$, and $\cT = \cT_\cA(\hat p)$ where $\hat p(y) = N(y) / j$;
}
Set $\hat y = \argmax_{y\in\cY} N(y)$;\\
\KwResult{Estimated mode $\hat y = \argmax \sum_{j\in[n]} \ind{Y_j=y}$ of $p$.}
\label{algo:adaptive-gen}
\end{algorithm}

\begin{algorithm}
\caption{Adaptive Coding Exhaustive Search with Huffman Coding}
\KwData{Set of classes $\cY=\brace{y_1, \ldots, y_m}$ endowed with $p\in\prob\cY$.}
    Initialize $N(y) = 0$ for all $y\in \cY$; Set $\cT$ a Huffman tree with $V_y = N(y)$;\\
    \For{$j\in[n]$}{
        Get new sample $Y_j\sim p$ and identify it by querying the entries of $c_\cT(Y_j)$;\\
        Update $N(Y_j) = N(Y_j) + 1$, and the Huffman tree $\cT$ with Algorithm \ref{algo:Vitter};
}
Set $\hat y = \argmax_{y\in\cY} N(y)$;\\
\KwResult{Estimated mode $\hat y = \argmax \sum_{j\in[n]} \ind{Y_j=y}$ of $p$.}
\label{algo:adaptive}
\end{algorithm}

\begin{theorem}[Adaptive entropic search performance]
    \label{thm:adaptive}
	Given a $C$-balanced coding scheme (Definition \ref{def:balanced}), in order to fully identify $n$ samples $(Y_j)$ following the adaptive coding strategy of Algorithm~\ref{algo:adaptive-gen}, one needs on average $\E[T]$ queries, where 
    \begin{equation}
        \label{eq:huffman-queries}
        nH(p) \leq \E[T] \leq  Cn (H(p) + 1) +  28 Cm\log_2(n) + 6 Cm + m^2,
    \end{equation}
    and $H(p):=-\sum_{y\in \cY} p(y)\log_2(p(y))$ is the entropy of the distribution $p\in\prob{\cY}$.
    In particular, in the case of Huffman coding, Algorithm \ref{algo:adaptive} yields
    \[
        nH(p) \leq \E[T] \leq  2n (H(p) + 1) + o(n).
    \]
\end{theorem}

As shown by Theorem \ref{thm:adaptive}, to reach a probability of error smaller than $\delta$, Algorithm~\ref{algo:adaptive-gen} needs $T = \ln(1/\delta) \alpha$ queries with $\alpha  =C\Delta_2^{-2}(H(p)+1)$ up to higher-order terms, with $C=1$ for Shannon coding and $C=2$ for Huffman coding.
This is an improvement in most cases over the coefficient $\alpha =\Delta_2^{-2}  \ceil{\log_2(m)}$ of Algorithm~\ref{algo:exhaustive}.
Nonetheless, both {\em Algorithms~\ref{algo:exhaustive} and~\ref{algo:adaptive-gen} precisely estimate $p(y)$ for each class $y$, which is more information than needed if one only wants the mode of $p$,} leaving room for improvements.

\subsection{Proofs}
\label{proof:huffman}

This subsection is devoted to the proof of Theorem~\ref{thm:adaptive}.
It is well known from information theory that one has to ask at least $H(p)$ questions on average to be able to identify $Y\sim p$ \citep{Shannon1948}, which explains the lower bound.
We need to prove the upper bound.

Recall that Algorithm~\ref{algo:adaptive-gen} identifies each $Y_i$ by following a $C$-balanced tree with respect to the empirical distribution $\hat p_i$ of the previously observed samples $(Y_j)_{j\leq i}\subset \cY$ as values.
Let us denote by $T_n$ the number of queries made to identify all the $(Y_i)$ for $i\leq n$.
At round $i$, our search procedure associated with the probability distribution $\hat p_i$ identifies each element $y$ in $\cY$ such that $\hat p_i(y)\neq 0$ with at most $C\ceil{-\log_2(\hat p_i(y))}$ queries as per Definition \ref{def:balanced}.
When $y$ is such that $\hat p_i(y)=0$, we identify it with at most $m$ queries.
Hence, at round $i+1\in[n]$, $Y_{i+1}\sim p$ is identified with at most the following number of questions on average
\[
   C + \E_{y\sim p}\left[  -\ind{\hat p_i(y) \neq 0}\cdot C\log_2(\hat p_i(y)) + \ind{\hat p_i(y) = 0}\cdot m\right],
\]
where $\hat p_i$ is our estimate of $p$ after the complete identification of $Y_i$, hence the one used for the queries on $Y_{i+1}$.
Note that in this setting, for each $y\in\cY$, there is at most a single round where we need to identify $Y_i = y$ while $\hat p_{i-1}(y) = 0$.
Recursively, this leads to
\[
      \E_{(Y_i)_{i\leq n}}\bracket{T_n} \leq Cn - C\sum_{i=1}^{n-1} \E_{(Y_j)}\left[ \sum_{y\in\cY}p(y) \ind{\hat p_i(y) \neq 0}\cdot\log_2(\hat p_i(y))\right] + m^2.
\]
Let us assume without loss of generality that $p(y)$ is never $0$, as the terms for which $p(y)=0$ do not contribute to the sum.
Now we split this equation into two parts: a first part where one does not have a tight control of the empirical distribution, but that is really unlikely and will not contribute much to the full picture; and a part where the empirical distribution concentrates towards its real mean,
\begin{align*}
      &- \sum_{i=1}^{n-1} \E_{(Y_j)}\left[ \sum_{y\in\cY}p(y)\ind{\hat p_i(y)\neq 0}\log_2(\hat p_i(y))\right] = 
      - \sum_{i=1}^{n-1}  \sum_{y\in\cY}p(y) \E_{(Y_j)}\left[\ind{N_{y,i}\neq 0}\log_2\left(\frac{N_{y,i}}{i}\right)\right] 
      \\&= \sum_{i, y} p(y) \E_{(Y_j)}\left[
      1_{\{|\frac{N_{y,i}-ip(y)}{ip(y)}|\geq \frac{1}{2}\}}\ind{N_{y,i}\neq 0}\log_2\left(\frac{i}{N_{y,i}}\right) 
      +
      1_{\{|\frac{N_{y,i}-ip(y)}{ip(y)}|<\frac{1}{2}\}}\log_2\left(\frac{i}{N_{y,i}}\right)
      \right],
\end{align*}
where $N_{y,i} = \sum_{j\leq i} \ind{Yj = y}$ is the number of times we have seen $y$ in the first $i$ samples.
The first term corresponds to a highly unlikely event, which we prove in Appendix \ref{app:exhaustive}.

\begin{restatable}{lemma}{technicalqueryone}
    With $N_{i,y}$ denoting the empirical count $\sum_{j\in[i]} \ind{Y_j = y}$,
    \[
          -  \sum_{i=1}^{n-1}  \sum_{y\in\cY}p(y)\E_{(Y_j)}\left[1_{\{|\frac{N_{y,i}-ip(y)}{ip(y)}|\geq \frac{1}{2}\}}\ind{N_{y,i} \geq 1}\log_2\left(\frac{N_{y,i}}{i}\right)\right] 
          \leq 22 m \log_2(n),
    \]
\end{restatable}

We now consider the second term.
Let us extract the scaling in $H(p)$ inherent to entropy coding.
To that end, we rewrite it as
\[
      -  \sum_{i=1}^{n-1}  \sum_{y\in\cY}p(y)  \E_{(Y_j)}\bigg[1_{\{|\frac{N_{y,i}-ip(y)}{ip(y)}|<\frac{1}{2}\}}\bigg(\log_2\left(p(y)\right) + \log_2\left(\frac{N_{y,i}}{ip(y)}\right) \bigg)\bigg].
\]
The first term presents the desired scaling
\[
      -  \sum_{i=1}^{n-1}  \sum_{y\in\cY}p(y)  \E_{(Y_j)}\left[1_{\{|\frac{N_{y,i}-ip(y)}{ip(y)}|<\frac{1}{2}\}} \log_2\left(p(y)\right) \right]
      \leq  -  \sum_{i=1}^{n-1}  \sum_{y\in\cY}p(y) \log_2\left(p(y)\right) \leq nH(p).
\]
Finally, we deal with the rightmost logarithm, using the Taylor series of the logarithm to show concentration for the empirical mean of the logarithm.
The details are provided in Appendix~\ref{app:exhaustive}.

\begin{restatable}{lemma}{technicalquerytwo}
    With $N_{i,y}$ denoting the empirical count $\sum_{j\in[i]} \ind{Y_j = y}$,
    \[
     - \ln(2)\E_{(Y_j)}\left[1_{\{|\frac{N_{y,i}-ip(y)}{ip(y)}|<\frac{1}{2}\}}\log_2\left(\frac{N_{y,i}}{ip(y)}\right)  \right] 
     \leq 4 m (\ln(n)+1).
    \]
\end{restatable}

Collecting the different pieces together, we find the upper bound,
\begin{align*}
      &  \E_{(Y_i)_{i\leq n}}\bracket{ T_n} \leq  Cn - C\sum_{i=1}^{n-1} \E_{(Y_j)}\left[ \sum_{y\in\cY}p(y) \ind{\hat p_i(y) \neq 0}\cdot\log_2(\hat p_i(y))\right] + m^2 \\&
     \leq Cn (1+H(p))+ m^2 + 
       22Cm\log_2(n) +\frac{4Cm}{\ln(2)}(\ln(n) +1) + m^2\\&
       \leq 
       Cn (1+H(p))+
      28 C m\log_2(n) + 6Cm + m^2.
\end{align*}

\note{See if we can have a similar bound in high-probability.}

\section{Truncated Search}
\label{sec:truncated}

In this section, we improve upon the previous search procedures using the following key observation.
When estimating the empirical mode of a batch by identifying samples following a Huffman tree, one can stop the search procedure roughly when reaching the depth $D=\abs{\log_2(p(y_1))}$ of the mode, resulting in about $\abs{\log_2 p(y_1)}$ queries per sample on average, rather than $H(p)$ queries when trying to fully identify each sample.

\subsection{Coarse Sufficient Statistics}
We introduce the concept of admissible partitions, which provide {\em sufficient statistics for mode estimation that are weaker than the full empirical distribution $\hat p$,} as well as the concept of $\eta$-admissible partition, which will be useful to build statistically and computationally efficient algorithms.
Recall that a partition $\cP$ of $\cY$ is a subset of $2^\cY$, such that for any $\brace{S_1, S_2} \subset \cP$, $S_1 \cap S_2 = \emptyset$, and $\cup_{S\in\cP} S = \cY$.
Here and throughout the text, for $p\in\prob{\cY}$ and $S\subset \cY$, we define $p(S):= \sum_{y\in S} p(y)$.
  
\begin{definition}[Admissible partitions]
    \label{def:admissible}
  An {\em admissible partition} of $\cY$ with respect to $p\in\prob\cY$ refers to a partition $\cP\subset 2^\cY$ such that
  \begin{equation}
	  \label{eq:admissible}
        \argmax_{S\in\cP} p(S) = \brace{y^*}, \qquad\text{where}\qquad y^* = \argmax_{y\in\cY} p(y).
  \end{equation}
  For $\eta > 0$, an {\em $\eta$-admissible partition} of $\cY$ with respect to $p$ is similarly defined as an admissible partition $\cP\subset 2^\cY$ for which all sets $S \in \cP$ such that $p(S) \geq \eta$ are singletons, and at most a single set $S$ verifies $p(S) < \eta/2$.
  Formally,
  \begin{equation}
	  \label{eq:eta-admissible}
    \card{\brace{S\in\cP\midvert p(S) < \eta / 2}} \leq 1 \qquad\text{and}\qquad
	  \forall\,S\in\cP,\quad{p(S) \geq \eta \quad\Rightarrow\quad |S| = 1}.
  \end{equation}
\end{definition}

The interest of admissible partitions comes from the fact that if $\cP$ is an admissible partition with respect to both $p$ and $p'\in\prob\cY$, then $p$ and $p'$ have the same modes. 
In particular, if $\cP$ is an admissible partition for the empirical distribution $\hat p$ of some batch of samples, the set $S\in \cP$ with the largest mass $\hat p(S)$ is a singleton containing the empirical mode of the batch.
On the computational side, when $p$ is known,   $\eta$-admissible partitions are easy to build, which contrasts with the NP-hardness of finding an admissible partition of minimal cardinality, or of finding a set $S$ that maximizes $p(S)$ under the constraint $p(S) \leq p(y^*)$.\footnote{
    The latter is equivalent to the knapsack problem \citep{Mathews1896}, while the partition problem \citep{Korf1998} can be reduced to the former with the following construction.
    Consider a list $p_0$ of positive integers and include an element $p_*$ equal to half the sum of $p_0$ plus an infinitesimal quantity. 
    Next, normalize all elements to transform the new list into a probability vector $p$ whose elements add up to one.
    Checking if $p_0$ can be partitioned into two lists $S_1$ and $S_2$ that sum to the same value is equivalent to determining whether there exists an admissible partition of $p$ of cardinality three.
}

\subsection{Adaptive Truncated Search}
Building upon Definition \ref{def:admissible}, Algorithm~\ref{algo:batch-search} efficiently constructs an $\eta$-admissible partition $\cP$ of $\cY$ with respect to the empirical frequencies $\hat p$ of a batch of samples $(Y_j)_{j\in[n]}$.
It uses a predefined binary tree $\cT$, and  takes two parameters $\gamma, \epsilon\in [0,1]$ that define $\eta = \gamma\hat p(\hat y) - \epsilon$, where $\hat y$ is the mode of $\hat p$. 
The algorithm starts with the trivial partition $\cP = \brace{\cY}$, and recursively refines it by splitting the set $S_*$ with the greatest empirical mass until $S_*$ is a singleton, which must then be equal to $\{\hat y\}$.
It then keeps splitting non-singleton sets of mass strictly greater than $\gamma\hat p(\hat y) - \epsilon$ until there are no such sets left.
The splitting is done using the tree $\cT$ as follows: we identify each node $V$ in the tree with the set of all the elements that map to its descendent leaves $S(V) = \brace{y\in\cY\midvert V_y \lhd V}$. 
At each time step, the sets $S$ of the current partition correspond to nodes $(V_S)$ of $\cT$, and the set $S_* = S(V_*)$ that has to be split is replaced in the partition by its two children $S(l(V_*)), S(r(V_2))$ in $\cT$. 
This consumes $N(S_*) = \sum_{j\in[n]} \ind{Y_j \in S_*}$ queries to identify which sample belongs to $S_1$ and which to $S_2$.
At the end, a Huffman scheme is applied to merge sets into an $\eta$-admissible partition, and re-balance the tree $\cT$ so that all sets in the partition are at a similar depth, roughly equal to $\log_2(2/\eta)$, in the new tree.
This re-balancing keeps the structure of the ``sub-trees'' below the nodes corresponding to the sets of the partition intact.
  \note{maybe clarify what we mean by this (we shuffle at the level of the partition, but we keep the tree below the partition), and explain a little more why it does what we say it does in the proofs}
In addition, the algorithm identifies which sample belongs to which set of $\cP$, as well as the empirical mode $\hat y$.

\begin{algorithm}[ht]
\caption{Batch Tree Rebalancing}
\KwData{Set $\cY=\brace{y_1, \ldots, y_m}$, binary tree $\cT$, $n$ samples $(Y_j)$, parameters $\gamma, \epsilon \in \bR$.}
    Set $V_*$ the root of the tree $\cT$, $S_* = \cY$, $N(S_*) = n$;\\
    Set $\cS = \brace{(V_*: n)}$ built as a heap, $\cL = \{\}$ an empty list, and $C = -\infty$;\\
    For all $S$ and $V$, we denote $N(S) = \sum_{j\in[n]}\ind{Y_j\in S} \in \bN$ and $S(V) = \brace{y\in\cY\midvert V_y\lhd V} \subset \cY$.\\
    \While{the heap $\cS$ is non-empty and $N(S_*) \geq C$}{
        Set $V_* = \argmax_{V\in\cS} N(S(V))$ by popping it out of the heap $\cS$; set $S_* = S(V_*)$;\\
        \uIf{$V_*$ is a leaf}{
            If it was the first encountered leaf, set $\{\hat y\} := S_*$, and refine $C := \gamma N(\{\hat y\}) - \epsilon n$;\\
            Add $V_*$ to the list $\cL$;
        }
        \Else{
            Make $N(S_*)$ queries to get all the information on $(\ind{Y_j \in S(V)})_{j\in[n]}$ for $V\in\brace{l(V_*), r(V_*)}$;\\
            Insert each child $V \in \brace{l(V_*), r(V_*)}$ into the heap $\cS$ with value $N(S(V))$;
        }
    }
    Add all the remaining elements of the heap to the list $\cL$;\\
    Apply Huffman's scheme, Algorithm \ref{algo:huffman}, to the list $\cL$ to rebalance the top of the tree $\cT$;\\
\KwResult{A tree $\cT$ containing an $(\gamma \hat p(\hat y) - \epsilon)$-admissible partition for $\hat p$, and the empirical mode $\hat y$.}
\label{algo:batch-search}
\end{algorithm}

This suggests a new way to tackle Problem \ref{Problem}: given a batch of samples, Algorithm~\ref{algo:batch-search} yields an admissible partition with respect to the empirical distribution $\hat p$, and in particular identifies its empirical mode $\hat y$, which is the best possible mode estimate. It does not need to fully identify each sample to do so, unlike Algorithms~\ref{algo:exhaustive} and~\ref{algo:adaptive}.
The number of queries consumed by Algorithm~\ref{algo:batch-search} depends on the tree $\cT$ that it takes as input: upon reaching the final partition $\cP$, the algorithm will have required $D$ queries for each sample belonging to a given set $S\in\cP$, where $D$ is the depth of the node associated with $S$ in the tree $\cT$.
To minimize the expected number of queries needed, $\cT$ should be a partial Huffman tree with respect to the distribution $p$; this would result in $D \leq C\ceil{-\log_2(2/p(y_1))}$ for each $S\in\cP$, with $C = 2$ as per Lemma \ref{lem:depth_in_Huffman_tree}.
We do not have a priori access to the distribution $p$, but we can learn the structure of such a tree over several rounds, with a slack parameter $\epsilon_r$ to account for the gradually increasing precision of our estimates. 
This is Algorithm~\ref{algo:truncated}, whose guarantees are provided by Theorem~\ref{thm:truncated}.

\begin{algorithm}
\caption{Truncated Search} 
\KwData{Set $\cY=\brace{y_1, \ldots, y_m}$, samples $(Y_s)$, scheduling $(n_r)\in\bN^\bN$ and $(\epsilon_r)\in\bR^\bN$.}
    Set $\cT$ a binary search tree;\\
    \For{$r\in\bN$}{
    Get $n_r$ fresh samples $(Y_{j,r})_{j\in[n_r]}$ defining an empirical probability $\hat p_r$;\\
    Run Algorithm~\ref{algo:batch-search} with tree $\cT$, $\gamma=1$ and $\epsilon=\epsilon_r$ to obtain the empirical mode $\hat y_r$ and update $\cT$;
    }
\KwResult{Return the last $\hat y_r$ as the estimated mode.}
\label{algo:truncated}
\end{algorithm}

\begin{theorem}[Truncated search performance]
    \label{thm:truncated}
     For any $\delta>0$, 
     the number of queries $T$ needed by Algorithm~\ref{algo:truncated} run with the schedulings $n_r := 2^r$ and $\epsilon_r := \frac{1}{4m}\paren{\frac{2}{3}}^{\frac{r}{2}}$ to correctly identify the mode with probability at least $1-\delta$ satisfies
     \begin{equation}
         \E[T] \leq 8\ln(1/\delta)\Delta_{2}^{-2} \ceil{\log_2\paren{\frac{4}{p(y_1)}}} + o(\ln(1/\delta)).
     \end{equation}
\end{theorem}

Algorithm~\ref{algo:truncated} is a clear amelioration over the previous search algorithms: the probability of error remains the same as a function of the number of samples, while reducing the average number of queries per sample from $H(p)$ to roughly $\abs{\log_2 p(y_1)}$, going from average code length to minimal code length.
Nonetheless, there is still room for improvement.
Algorithm~\ref{algo:truncated} identifies the empirical mode of the entire batch with absolute certitude.
This contrasts with the main takeaway from the bandit literature: {\em one should leverage confidence intervals to build statistically efficient algorithms so as to avoid spending queries solely to rule out highly unlikely events.}
We explore this intuition in the next section.
Note that replacing the Huffman code from Algorithm~\ref{algo:batch-search} by any $C$-balanced code would result in a modified bound 
\begin{equation*}
         \E[T] \leq 4C\ln(1/\delta)\Delta_{2}^{-2} \ceil{\log_2\paren{\frac{4}{p(y_1)}}} + o(\ln(1/\delta)).
     \end{equation*}
 in Theorem~\ref{thm:truncated} (in particular, using a Shannon code would yield $C=1$).   

\subsection{Proofs}

This subsection provides the key elements for the proof of Theorem \ref{thm:truncated}.
Our strategy is the following.
\begin{itemize}
    \item Find a likely event $\cA$ where the algorithm behaves as desired.
	\item Show that its complement ${}^c\cA$ is unlikely enough that the additional queries it elicits are asymptotically negligible.
\end{itemize}

Let $r \geq 2$.
We let $\cT_r$ be the updated tree at the end of round $r$, and $\cP_r$ be the corresponding $\eta_r$-admissible partition where $\eta_r = \hat p_r(\hat y_r) - \epsilon_r$.
Here $\hat p_r$ is the empirical distribution of the samples at round $r$, and $\hat y_r$ is the corresponding empirical mode.
Let us define the event
\begin{equation}
    \label{eq:ass}
    {\cal A} = \brace{\forall\, S\subset\cY, \quad \max\brace{|\hat p_{r-1}(S) -p(S)|, |\hat p_{r}(S) -p(S)|} \leq\frac{\epsilon_{r-1}-\epsilon_{r}}{4}}.
\end{equation}
A union bound, detailed in Appendix \ref{app:truncated}, shows that ${\cal A}$ is likely to happen as the number of rounds increases, as stated in the following lemma.

\begin{restatable}{lemma}{trunctechnicalone}
  The event $\cA$ defined by \eqref{eq:ass} satisfies
    \begin{equation}\label{eq:bound_on_event_Ac}
     \bP\left({}^c\cA\right)\leq 2^{m+1} \exp\left(-\frac{n_{r-1}}{2}\left(\frac{\epsilon_{r-1}-\epsilon_r}{4}\right)^2\right).   
    \end{equation}
\end{restatable}

We claim that under the event $\cA$, the $\eta_{r-1}$-admissible partition obtained at the end of round $r-1$ stays admissible at round $r$, namely, we have the implication
\begin{equation}
  \label{eq:trunc-admissible}
  	\cA \quad\subset\quad
    \brace{\text{All sets } S\in \cP_{r-1} \text{ such that } \hat p_r(S)>  \hat p_r(\hat y_r) - \epsilon_r \text{ are singletons}}.
\end{equation}
\begin{proof}[Proof of Equation \eqref{eq:trunc-admissible}]
Indeed, $\cA$ implies that $|\hat p_{r-1}(S) -\hat p_r(S)| \leq (\epsilon_{r-1}-\epsilon_{r}) / 2$ and as a consequence, for any set $S\in \cP_{r-1}$ of cardinality at least $2$, the fact that $\cP_{r-1}$ is $(\hat p_{r-1}(\hat y_r) - \epsilon_{r-1})$-admissible with respect to $\hat p_{r-1}$ leads to
\begin{align*}
    \hat p_r(S)
    &\leq \hat p_{r-1}(S) + \frac{\epsilon_{r-1}-\epsilon_{r}}{2} 
    < \hat p_{r-1}(\hat y_{r-1}) -\epsilon_{r-1} + \frac{\epsilon_{r-1}-\epsilon_{r}}{2} 
    \\&\leq \hat p_{r}(\hat y_{r-1}) - \epsilon_r \leq  \hat p_{r}(\hat y_{r}) - \epsilon_r,
\end{align*}
Hence, under $\cA$, all sets $S\in \cP_{r-1}$ such that $\hat p_r( S) > \hat p_r(\hat S_r) - \epsilon_r$ are singletons. 
\end{proof}

Let us assume that $\cA$ is realized --then all sets $S\in \cP_{r-1}$ are either singletons or such that $\hat p_r(S)\leq   \hat p_r(\hat y_r) - \epsilon_r$.
Hence, when Algorithm~\ref{algo:batch-search} is applied to the $n_r$ samples of round $r$, a sample $Y_{i,r}$ belonging to some set $S\in \cP_{r-1}$ only consumes a number of queries smaller\footnote{The number of queries needed can be strictly smaller than the depth of $S$ if some parent set $S'$ of $S$ is such that $\hat p_{r-1}(S') > \hat p_{r-1}(\hat y_{r-1}) - \epsilon_{r-1}$ but $\hat p_{r}(S') \leq \hat p_{r}(\hat y_{r}) - \epsilon_{r}$.} or equal to the depth $D(S)$ of $S$ in the tree $\cT_{r-1}$.
Since $\cT_{r-1}$ was obtained by applying Huffman's scheme to the sets of $\cP_{r-1}$ at the end of Algorithm~\ref{algo:batch-search} to re-balance the tree $\cT_{r-2}$, we can leverage Lemma \ref{lem:depth_in_Huffman_tree} to bound the depth $D(S)$.
Since $\brace{y_1} \subset \cY$, when $\cA$ is realized
\[
    p(y_1)- \hat p_{r-1}(\hat y_{r-1})\leq p(y_1)- \hat p_{r-1}(y_1) < \frac{\epsilon_{r-1}-\epsilon_{r}}{4} < \epsilon_r.
\]
As $\cP_{r-1}$ is $(\hat p_{r-1}(\hat y_{r-1}) - \epsilon_{r-1})$-admissible, necessarily each $S\in\cP_{r-1}$ (except at most one) is such that 
\[
    \hat p_{r-1}(S)\geq \frac{\hat p_{r-1}(\hat y_{r-1}) - \epsilon_{r-1}}{2} > \frac{p( y_{1}) - \epsilon_{r}- \epsilon_{r-1}}{2}\geq \frac{p( y_{1}) - 2\epsilon_{r-1}}{2} \geq \frac{p( y_{1}) }{4} ,
\]
where we use the fact that  $\epsilon_{r-1} < 1 / 4m \leq p(y_1) / 4$.
Applying Lemma \ref{lem:depth_in_Huffman_tree}, we deduce that 
\begin{equation}
  \label{eq:depth-trunc}
  	D(S)\leq 2\ceil{\log_2(4/p(y_1))}.
\end{equation}
If $S$ is such that $\hat p_{r-1}(S)< \frac{\hat p_{r-1}(\hat y_{r-1}) - \epsilon_{r-1}}{2}$ (and there can be only one such set), then it must share a parent in $\cT_{r-1}$ with some other set $S'$, from which we deduce that its depth follows the same bound.
Hence the total number of queries consumed at round $r$, assuming that $\cA$ is realized, is at most $n_r 2\ceil{\log_2(4 / p(y_1))}$.
If the event $\cA$ is not realized, we can roughly upper bound the number of queries spent per sample by $m$, hence we can upper bound the total number of queries by $n_r m$.
In the special case $r=1$, we similarly need at most $n_r m$ queries.
A few lines of derivations provided in Appendix \ref{app:truncated} lead to the following lemma.

\begin{restatable}{lemma}{truncatedtechnicaltwo}
  When running Algorithm \ref{algo:truncated} with the schedule of Theorem \ref{thm:truncated}, the expected number of queries needed for round $r$ satisfies
  \[
	  \E\bracket{T_r} \leq 2^r \paren{ \ceil{\log_2\paren{\frac{4}{p(y_1)}}} +
	  m2^{m+1} \exp\paren{-\paren{\frac{4}{3}}^r  \frac{C}{m^2} }}
  \]
  for some constant $C>0$.
\end{restatable}

At the end of round $r$, the algorithm outputs the empirical mode $\hat y_r$ of the $n_r$ samples $(Y_{i,r})_{i\in[n_r]}$ defining the empirical probability $\hat p_r$.
We know from Theorem \ref{thm:empirical-mode} that its probability of error $\delta_r$ is bounded by
\[ 
    \delta_r  \leq \exp\paren{-n_r\Delta_{2}^2 } = \exp\paren{-2^r\Delta_{2}^2}. 
\]
Now let $\delta>0$, and define the round
\[
  r_\delta = \min\brace{r\midvert \exp(-n_r\Delta_2^2) \leq \delta}
	= \min\brace{r\midvert n_r \geq \Delta_2^{-2}\ln(1/\delta)}
\]
By construction, $n_{r_\delta} = 2n_{r_\delta-1} \leq  2\ln(1/\delta)\Delta_{2}^{-2}$.
At the end of round $r_\delta$, the probability of outputting the correct class is at least $1-\delta$, and the total number of expected queries used so far by Algorithm \ref{algo:truncated} is at most 
\begin{align*}
  \sum_{r=1}^{r_\delta} \E[T_r]  & \leq 
\sum_{r=1}^{r_\delta}  n_r \left(2  \ceil{\log_2\left(\frac{4}{p(y_1)}\right)} + 
m2^{m+1} \exp\left(-\left(\frac{4}{3}\right)^r  \frac{C}{m^2} 
\right) \right)  \\
&\leq   2\ceil{\log_2\left(\frac{4}{p(y_1)}\right)} \sum_{r=1}^{r_\delta}  n_r  + 
\sum_{r\geq 1} 2^r m 2^{m+1}  \exp\left(-\left(\frac{4}{3}\right)^r  \frac{C}{m^2} 
\right)  \\
& \leq 4 n_{r_\delta} \ceil{\log_2\left(\frac{4}{p(y_1)} \right)} + \tilde{C} \leq 
  8\ln(1/\delta)\Delta_{2}^{-2} \ceil{\log_2\left(\frac{4}{p(y_1)} \right)} + o(\ln(1/\delta)).
\end{align*}
This completes the proof.

\section{Bandit-Inspired Elimination}
\label{sec:elim}

The core idea of this section is to successively eliminate candidate classes $y\in\cY$ in order to lower the required number of queries compared to the Exhaustive Search Algorithm from Section~\ref{sec:exhaustive}, which can be done by adapting and improving the seminal Successive Elimination Algorithm from \citet{Evendar2006}, resulting in Algorithm~\ref{algo:elimination}.

\begin{algorithm}
\caption{Elimination}
\KwData{Set $\cY$, probability $p\in\prob\cY$, samples $\{Y_r\}_{r\in\bN}$, schedule $\sigma:\bN\times\bR^\cY\to\bR^\cY$.}
    Set $S_e = \emptyset$ the set of eliminated guesses, $r = 0 $, $N(y) = 0 \in \bR^m$;\\
    \While{$\cY \backslash S_e$ is not a singleton}{
    Set $r \leftarrow r+1$ and query $\ind{Y_r\in S_e}$;\\
    \uIf{$Y_r \notin S_e$}{
    Identify $Y_r$ with a Huffman tree adapted to empirical counts $N(y)$ on $\cY\setminus S_e$;\\
    Update $N(Y_r) = N(Y_r) + 1$, and set the empirical distribution $\hat p_r(y) \propto N(y)$ accordingly; \\
    Set $S_e = S_e \cup \braced{y \in \cY\setminus S_e\vert \hat p_r(y) + \sigma_y(r,\hat p_r) < \max_{z\in\cY\setminus S_e} \hat p_r(z)}$;
    }
    }
\KwResult{Estimate mode $\hat y$ as the only element of $\cY \backslash S_e$.}
\label{algo:elimination}
\end{algorithm}

Algorithm~\ref{algo:elimination} takes as input a generic parameter function $\sigma$ and works as follows: at any time, it maintains a set of eliminated classes $S_e\subset \cY$ that are deemed unlikely to be the mode.
At the start of round $r$, the algorithm tests whether the new sample $Y_r$ belongs to $S_e$; if not, it is identified with a Huffman tree adapted to the empirical distribution $\hat p_r$ on $\cY\backslash S_e$, which is then updated.\footnote{There is a minor abuse of notations in the description of Algorithm~\ref{algo:elimination}: in the expression ``$\sigma_y(r,\hat p_r)$", the distribution $\hat p_r$ on $\cY\backslash S_e$ is implicitly seen as a distribution on $\cY$ by setting $\hat p_r(y) =0$ if $y\in S_e$ so that it is a valid argument for $\sigma$.
}
Any class $y$ such that its empirical mass $\hat p_r(y)$ satisfies $p_r(y) + \sigma_y(r,\hat p_r) < \max_{z\in\cY\setminus S_e} \hat p_r(z)$ is finally added to the set of eliminated classes $S_e$.
For a well-chosen parameter function $\sigma$, Algorithm~\ref{algo:elimination} has a high probability of identifying the true mode, as stated in Theorem~\ref{thm:elimination}.

\begin{theorem}
\label{thm:elimination}
    Let $\delta > 0$, and define the elimination scheduling $\sigma:\bN\times\bR^\cY\to\bR^\cY$ as
    \begin{equation}
    \label{eq:crit-elim}
      \sigma_y(r, \hat p) = \sqrt{\frac{c\max_z \hat p_r(z)\ln(\pi^2 m r^2/\delta)}{r}}, \qquad\text{with}\qquad c = 24.
    \end{equation}
    With probability $1-\delta$, Algorithm~\ref{algo:elimination} terminates, identifies the right mode and has consumed less than $T$ queries, where, conditionally to this successful event,
    \[
	  \frac{\E[T]}{\ln(1/\delta)} \leq  C_1 \frac{p(y_1)}{\nabla_{2}^{2}} + C_2\sum_{i\in[m]} p(y_i) \frac{p(y_1)}{\nabla_i^2}  \ceil{\log_2\paren{\frac{1}{p(y_i)}}} + o(1),
    \]
    for two universal constants $C_1, C_2$, $\nabla_i := p(y_1) - p(y_1)$ if $i\neq 1$ and $\nabla_1 = \nabla_2$.
\end{theorem}
As shown in Proposition~\ref{prop:comparison_Delta} below, whose proof can be found in Appendix~\ref{app:elim}, $\nabla_{i}^{-2}p(y_1)$ is equal to $\Delta_{i}^{-2} $ up to a multiplicative factor as long as $p(y_1)$ is bounded away from $1$ --this is no constraint in most interesting use cases, as Problem~\ref{Problem} gets easier the closer $p(y_1)$ gets to $1$.
Hence we see that Theorem~\ref{thm:elimination} is a clear improvement upon the Exhaustive Search Algorithms \ref{algo:exhaustive} and \ref{algo:adaptive}.
In essence, one requires roughly $\Delta_{i}^{-2}\log(1/\delta)$ samples to know with certainty $1-\delta$ that class $y_i$ is not the mode. 
As the Exhaustive Search Algorithm from Section~\ref{sec:exhaustive} never eliminates any class, the number of samples it requires to reach accuracy $1-\delta$ is conditioned by the second most likely class $y_2$, the one that is hardest to correctly dismiss.
Conversely, Algorithm~\ref{algo:elimination} eliminates with high probability the class $y_i$ after having seen about $\log(1/\delta) \Delta_{i}^{-2}$ samples, which translates to an improved asymptotic expected number of required queries as a function of $\delta$.
However, {\em Algorithm~\ref{algo:elimination} does not leverage the coarse statistics} and search truncation mechanism presented in Section \ref{sec:truncated}, leaving room for further improvement. We combine the elimination mechanism of  Algorithm~\ref{algo:elimination} and the search truncation technique of Algorithm~\ref{algo:truncated} in Section~\ref{sec:set-elim}, after the remainder of this section which is dedicated to the discussion of Algorithm~\ref{algo:elimination} and the proof of Theorem~\ref{thm:elimination}.

\begin{restatable}{proposition}{comparisondelta}
    \label{prop:comparison_Delta}
    Let $p \in \prob\cY$ with $p(y_1) > p(y_i)$ for all $i\in[m]$.
  We have the following inequality, with $\nabla_i := p(y_1) - p(y_i)$ and $\Delta_i^2 = - \ln(1 -(\sqrt{p(y_1)} - \sqrt{p(y_i)})^2)$,
    \[
        \frac{p(y_1)}{-\ln(1-p(y_1))} p(y_1)\nabla_i^{-2} \leq \Delta_i^{-2} \leq 4 p(y_1) \nabla_i^{-2}.
    \]
   In particular, if $p(y_1)$ is bounded away from $1$, then $- p(y_1) / \ln(1-p(y_1)$ is bounded away from $0$ and $\Delta_i^{-2} \simeq  p(y_1)\nabla_i^{-2}$ up to a multiplicative constants. 
\end{restatable}

\subsection{Design of the elimination schedule}
Algorithm \ref{algo:elimination} was inspired by the seminal algorithm of \citet{Evendar2006}.
Their Successive Elimination Algorithm uses Hoeffding's inequality to define $\sigma_y(r, p) = \sigma_r$ where $\sigma_r^2 \simeq \ln(1/\delta) / r$.
It is known that Hoeffding's inequality can be refined for sub-gamma variables with Bernstein's inequalities, which was the main motivation behind the UCB-V algorithm of \citet{Audibert2009}.
Bernstein's inequality can be seen as a convenient weak formulation of Chernoff's bounds, which in the case of Bernoulli variables are expressed in terms of  KL-divergences, and have motivated the KL-UCB algorithm of \citet{Capp2013}.
KL-based confidence intervals are defined as
\[
  \sigma_y(r, p) = \inf\brace{q\in[0,1]\midvert  \exp(-r D(p(y)\| p(y) \pm q)) \leq \delta}.
\]
While they arguably provide the best asymptotic results, they introduce computational overhead that we were keen to avoid in our algorithms.

Rather than trying to find the tightest confidence intervals, our definition of $\sigma_y(r, p)$ in Theorem \ref{thm:elimination} comes from a different perspective: the optimization of elimination times.
We know from Lemma \ref{lem:Chernoff}, a variant of Theorem \ref{thm:empirical-mode} provided in Appendix \ref{app:mode} based on a method of \citet{Cramer1938}, that we can not safely eliminate $y_i$ as a mode candidate before having made $r_i \simeq \Delta_i^{-2}\ln(1/\delta)$ observations of both $\ind{Y_j=y_i}$ and $\ind{Y_i=y_1}$.
The  Proposition~\ref{prop:comparison_Delta} below states that this is roughly the same as asking for 
\[
  r_i \simeq p(y_1)\nabla_i^{-2}\ln(1/\delta)
  = p(y_1)(p(y_1) - p(y_i))^{-2}\ln(1/\delta).
\]
Plugging this in the elimination criterion of Algorithm \ref{algo:elimination}, we would like to define $\sigma$ so as to ensure that with high probability,
\[
  \hat p_r(y_1) - \hat p_r(y_i) - \sigma_{y_i}(r, \hat p) \geq 0
\]
for any $r\leq r_i$.
Assuming that the empirical probability $\hat p_r$ converges fast enough to $p$, a confidence parameter  $\sigma$ defined as
\[
   \sigma_{y_i}(r, p) \simeq (p(y_1) - p(y_i))\sqrt{\frac{r_i}{r}} \simeq \sqrt{\frac{p(y_1) \ln(1/\delta)}{r}}
\]
satisfies this constraint.
Of course, we do not have a priori access to $p(y_1)$. Instead, we estimate it at round $r$ with $\max_{y\in \cY} \hat p_r(y)$.
Up to a few constants needed to account for union bounds, this leads to  $\sigma$ as defined in Theorem~\ref{thm:elimination}.
In comparison to the results of \citet{Evendar2006}, whose algorithm requires $O(\sum_i \nabla_i^{-2}\log(1/\delta))$ samples, we gain a factor $p(y_1) < 1$, thanks to which we reach the ideal scaling in $\Delta_i^{-2}$ as long as $p(y_1)$ is bounded away from $1$.

\subsection{Proofs}
In this section, we prove Theorem \ref{thm:elimination}.
Let $\hat p_r$ denotes the empirical probability of the first $r$ samples $(Y_i)_{i\leq r}$, and $S_r$ be the set $S_e$ of rejected classes, updated at each iteration $r$ by Algorithm \ref{algo:elimination}.
We define $\hat y_r = \argmax_{y\in\cY\backslash S_r} \hat p_r(y)$, and $\delta_r = \delta / \pi^2 m r^2$, and we consider the elimination criterion 
\[
  \hat p_r(y) + \sigma_r < \hat p_r(\hat y_r),\qquad
  \sigma_r = \sqrt{\frac{24\hat p_r(\hat y_r)\ln(1/\delta_r)}{r}}
\]
from Algorithm \ref{algo:elimination} with the schedule from Theorem \ref{thm:elimination}.

We have seen in the discussion of the previous subsection that the estimation of the unknown quantity $p(y_1)$ by $\hat p_r (\hat y_r)$ is a key component of the definition of $\sigma_r$.
In fact, it turns out that estimating $p(y_1)$ is much easier than finding the mode $y_1$, yet can help us with that second, harder task. The next two lemmas deal with this subproblem, and are crucial to our proof of Theorem~\ref{thm:elimination}.
Chernoff's bound for Bernoulli variables states that one needs roughly $\ln(1/\delta) / p(y)$ samples to get a good estimate of $p(y)$ up to multiplicative constants.
It leads to the following lemma, proved in Appendix \ref{app:elim}.
\begin{restatable}{lemma}{modeestone}
    \label{lem:mode-est}
    For any $c > 1$, let $\hat p_r$ be the empirical probability associated with $r$ random samples $(Y_i)_{i\leq r}$ independently distributed according to $p\sim \prob\cY$. It holds that
    \[
        \forall\, r\geq \frac{c+1}{(c-1)^2}\frac{1}{p(y)}\ln(1/\delta), \qquad \bP(\hat p_r(y) > cp(y)) \leq \delta
    \]
    and
    \[
        \forall\, r\geq \frac{c^2}{(c-1)^2}\frac{1}{p(y)}\ln(1/\delta), \qquad \bP(\hat p_r(y) < c^{-1}p(y)) \leq \delta.
    \]
\end{restatable}
While Lemma \ref{lem:mode-est} shows that one needs about $r \simeq \ln(1/\delta) / p(y_1) $ samples in order to get a good estimate of $p(y_1)$, it is not of much use by itself: as $p(y_1)$ is a priori unknown, we cannot compute this required number of samples.
We could naively estimate it as the first round $r$ such that $r\max_{y\in\cY}\hat p_r(y) \geq \ln(1/\delta)$.
The next lemma shows that this strategy actually works well.

\begin{restatable}{lemma}{modeesttwo}
    \label{lem:mode-est-crit}
    For any $r\in \bN_{\geq 1}$, $\delta >0$ and $c>1$, let $\hat y_r =\argmax_{y\in\cY}\hat p_r(y)$ and consider the event
    \[
      \cA_r  = \brace{r\hat p_r(\hat y_r)\leq c\ln(1/\delta)}.
    \]
    Then
    \[
       \forall\, r \leq \frac{2c^2 - c}{2c + 1 + \sqrt{1+8c}} \frac{1}{p(y_1)} \ln(1/\delta),\qquad \bP({}^c\cA_r) \leq m\delta,
    \]
    and
    \[
       \forall\, r \geq \frac{c^2}{c + 1 - \sqrt{1+2c}} \frac{1}{p(y_1)} \ln(1/\delta),\qquad \bP(\cA_r) \leq \delta.
    \]
\end{restatable}
Note that by combining Lemmas \ref{lem:mode-est} and \ref{lem:mode-est-crit}, one can derive an efficient $(0, \delta)$-PAC algorithm to get a good estimate of $p(y_1)$ up to  user-defined multiplicative constants: first, use Lemma~\ref{lem:mode-est} and the aforementioned constants to express the number $r_0$ of samples needed as a function of $p(y_1)$ and $\delta$.
This number cannot be explicitly computed by the user, as they do not have access to $p(y_1)$; however, one can use Lemma \ref{lem:mode-est-crit} to define some event $\cA$ associated to some constant $c$ such that once $\cA$ does not hold any more, $r$ is bigger than $r_0$ will high probability. For any such $r$, $\hat p_r(\hat y_r)$ will be a good estimate of $p(y_1)$.
Though we do not explicitly use such an algorithm, it is implicitly at the core of Theorem~\ref{thm:truncated}, as will become apparent in what follows.

As for the proof of Theorem \ref{thm:truncated}, we define likely events that ease the analysis.
Their definitions, and lower bounds on their probabilities, are provided by the following lemma, proven in Appendix \ref{app:elim}.
\begin{restatable}{lemma}{concentrationelim}
\label{lem:events_have_high_probability}
  Let us write $\tilde{\sigma}_r = \sqrt{3 p(y_1)\ln(1/\delta_r) / r}$, and define the events
  \begin{itemize}
	  \item $A_1 = \{\hat p_r(y_1) \geq  p(y_1)/2$  for all $r\geq 4\ln(1/\delta_r)/ p(y_1) \}$,
	  \item $A_2 = \{\hat p_r(\hat y_r) \leq 2p(y_1)$  for all $r\geq 4\ln(1/\delta_r) / p(y_1)\}$,
	  \item $A_3 = \{ |\hat p_r(y_i) - p(y_i) | \leq \tilde{\sigma}_r \}$ for all $r\geq 4 \ln(1/\delta_r) / p(y_1)$ and $i\in[m]$, 
	  \item $A_4 = \{ \sigma_r \geq \hat p_r(\hat y_r)$ for all  $r\leq 4 \ln(1/\delta_r) / p(y_1) \}$.
  \end{itemize}
  Then $\bP(A_i)\geq 1-\delta/6 $ for $i\in\{1,2,4\}$ and $\bP(A_3)\geq 1-\delta/3$.
\end{restatable}

Those events will allow us to guarantee the validity of Algorithm \ref{algo:elimination} with high probability.
Note that, due to a coincidence in our choice of constants, the event $A_2$ is in fact implied by $A_3$. We nonetheless keep $A_2$ as a separate event for increased readability. 
Intuitively, these events should be interpreted as follows: 
$A_1$ and $A_2$ ensure that $\hat p_r(\hat y)$ is roughly equal to $p(y_1)$ once the threshold $r = \frac{4}{p(y_1)} \ln(1/\delta_r)$, which is given by Lemmas \ref{lem:mode-est} and \ref{lem:mode-est-crit}, is reached.
Past this threshold, the quantity $\tilde{\sigma}_r = \sqrt{\frac{3 p(y_1)\ln(1/\delta_r)}{r}}$ acts as a good deterministic proxy for $\sigma_r/2 = \sqrt{\frac{6\hat p_r(\hat y)\ln(1/\delta_r)}{r}}$, and $A_3$ ensures that the empirical probability mass $\hat p_r(y_i)$ of each class $y_i$ is within an interval of width $ \tilde{\sigma}_r$ centered around the true mass $p(y_i)$.
Before this threshold, $A_4$ ensures that $\sigma_r$ is too large  for the elimination criterion to be satisfied, hence that no classes are eliminated before $r\leq \frac{4}{p(y_1)} \ln(1/\delta_r)$.
All of this combined keeps $y_1$ from being wrongly eliminated, and yields upper bounds on the elimination times of the classes $y_i$ with $i>1$. The next paragraphs will formalize these claims.
We start by showing that the mode is never wrongfully eliminated.

\begin{lemma}
\label{lem:y1_never_eliminated}
  When the events $(A_i)_{i\in[4]}$ hold, the mode $y_1$ is never eliminated.
\end{lemma}
\begin{proof}
  If $A_4$ holds, then for any $r \leq 4\ln(1/\delta_r)/p(y_1)$ we have $\hat p_r(y) + \sigma_r \geq \hat p_r(\hat y_r)$, hence, by definition of the elimination criterion, no classes can be eliminated at round $r$.

  Let us now consider any round $r\geq 4\ln(1/\delta_r)/p(y_1)$.
  Assume that $y_1$ has not yet been eliminated at the start of round $r$.
  The event $A_3$ implies
  \[  
  	\hat p_r(y_1) \geq p(y_1) - \tilde{\sigma}_r \geq p(y_i) - \tilde{\sigma}_r \geq \hat p_r(y_i) - 2\tilde{\sigma}_r,
  \]
  while the event $A_1$ leads to $\hat p_r(\hat y_r) \geq \hat p_r(y_1) \geq p(y_1)/2$, hence
  \[
  \sigma_r = \sqrt{24 \hat p_r(\hat y_r)\ln(1/\delta_r) / r} \geq \sqrt{12 p(y_1)\ln(1/\delta_r)/r} = 2\tilde{\sigma}_r.
  \]
  This means that the criterion $\hat p_r(y_1) + \sigma_r < \hat p_r(\hat y_r)$ cannot be satisfied,
  and thus $y_1$ cannot be eliminated at round $r$ if it was not at round $r-1$.
  Recursively, we deduce that $y_1$ is never eliminated.
\end{proof}

Let us now focus on the elimination of the other classes $y\neq y_1$. 

\begin{lemma}
\label{lem:elimination_time}
  When the event $(A_j)_{j\in[4]}$ hold, the class $y_i\in \cY\backslash \{y_1\}$ is eliminated no later than when
  \[
	r >  108  \frac{p(y_1)}{\nabla_i^2}\ln(1/\delta_r).
  \]
  Since $\ln(\delta_r) = \ln(1/\delta) + 2\ln(r) + c$, this defines an elimination time
  \begin{equation}
    \label{eq:elimination_time_in_proof}
	r(y_i) = 108\frac{p(y_1)}{\nabla_i^2} \ln(1/\delta) + o(\ln(1/\delta)).
  \end{equation}
\end{lemma}

\begin{proof}
  Let $i>1$, and assume that the events $(A_j)_{j\in[4]}$ hold.
  We know from Lemma \ref{lem:y1_never_eliminated} that our hypothesis implies that $y_1$ is never eliminated, hence that $\hat p_r(\hat y_r) \geq \hat p_r(y_1)$. Due to $A_2$, it also implies that 
  \[
  	\tilde{\sigma}_r =  \sqrt{\frac{3 p(y_1)\ln(1/\delta_r)}{r}} \geq \sqrt{\frac{3 \hat p_r(\hat y_r)\ln(1/\delta_r)}{2r}} = \frac{\sigma_r}{4}.
  \]
  Furthermore, $A_3$ implies that 
  \[ 
  	\hat p_r(\hat y_r) - \hat p_r(y_i) \geq p(y_1) - p(y_i) - 2\tilde{\sigma}_r.
  \]
  If the following inequality holds
  \[
    p(y_1) -p(y_i) > 6\tilde{\sigma}_r.
  \]
  then
  \[
  	\hat p_r(\hat y_r) - \hat p_r(y_i) \geq p(y_1) - p(y_i) - 2\tilde{\sigma}_r > 4 \tilde{\sigma}_r\ \geq \sigma_r.
  \]
  The lemma reduces to the characterization of $r$ such that $p(y_1) -p(y_i) > 6\tilde{\sigma}_r$.
\end{proof}

The two previous lemmas have shown that when $(A_i)_{i\in[4]}$ hold, $y_1$ is never eliminated and $y_i$ is eliminated at the latest when $r = 108(p(y_1) - p(y_i))^{-2}  p(y_1) \ln(1/\delta) + o(\ln(1/\delta))$. 
In particular, the algorithm ends and outputs the correct mode before $r_2$.
Using Lemma \ref{lem:events_have_high_probability}, this happens with probability at least
\begin{equation}
  \bP(\cap_{i\in[4]}\cA_i) \geq 1 - \bP({}^{c}A_1 \cup {}^{c}A_2 \cup {}^{c}A_3 \cup {}^{c}A_4) \geq 1- \sum_{i=1}^4  \bP({}^{c}A_i) \geq 1 - \delta,
\end{equation}

Regarding the number of expected queries, we use one query for each sample to check if it belongs to the eliminated set.
If it is not, i.e. $Y_i \notin S_r$, we expect, based on Theorem \ref{thm:adaptive}, to have to ask about $|\log_2(\tilde{p}(y))|$ $\leq \abs{\log_2(p(y))}$ queries to identify $Y_i = y$, where $\tilde p$ is the restriction of $p$ to the set $\cY\setminus S_r$. 

This explains the number of queries in Theorem \ref{thm:elimination},
\begin{align*}
  \E\bracket{T \midvert (A_j)_{j\in[4]}} &\lesssim \sum_{r\leq r(y_2)} \sum_{i\in[m]} p(y_i) (1 + \ind{y_i\notin S_r}|\log_2(p(y_i))|)
  \\ &\lesssim r(y_2) + \sum_{i\in[m]} r(y_i) p(y_i) |\log_2(p(y_i))|
  \\ &\simeq\parend[\Big]{\frac{p(y_1)}{\nabla_2^2} + \sum_{i\leq m} p(y_i)\frac{p(y_1)}{\nabla_i^2} |\log_2(p(y_i))|}\ln(1/\delta).
\end{align*}
The proof technicalities are provided in Appendix \ref{app:elim}, yielding the following lemma.

\begin{restatable}{lemma}{eliminationtech}
  \label{lem:elim-conclusion}
  In the setting of Theorem \ref{thm:elimination}, with the event $(A_i)_{i\in[n]}$ defined in Lemma \ref{lem:events_have_high_probability},
  \[
	\frac{\E\bracket{T | (A_j)_{j\in[4]}}}{\ln(1/\delta}
	  \leq 324  \frac{p(y_1)}{\nabla_2^2} + 216\sum_{i=1}^m p(y_i) \abs{\log_2(p(y_i))}\frac{p(y_1)}{\nabla_i^2} + o(1)
  \]
\end{restatable}

Lemma \ref{lem:elim-conclusion} ends the proof of Theorem \ref{thm:elimination} and explicits the constants $C_1$ and $C_2$.

\section{Set Elimination}
\label{sec:set-elim}

The intuition behind Algorithm~\ref{algo:truncated} is that given a Huffman tree with respect to $p$, one does not need to go far below the depth $-\log_2 p(y_1) + 1$ when searching for the mode (or equivalently, that classes with low probability can be grouped into sets of classes with probability roughly equal to $p(y_1)/2$). 
Meanwhile, the intuition behind Algorithm~\ref{algo:elimination} is that one can quickly eliminate low-probability classes to focus on likely candidates.
The combination of these ideas yields Algorithm~\ref{algo:set-elimination}, which outperforms both.
Classes are partitioned into sets of mass roughly equal to $p(y_1)/2$, which are eliminated as soon as they appear unlikely to be the mode. Those partitions must sometimes be redefined (as the sets might be of greater mass than initially estimated), but these reorderings are rare and do not impact asymptotic performance.
Hence, under the light of Proposition \ref{prop:comparison_Delta}, Algorithm~\ref{algo:set-elimination} requires roughly $\Delta_i^{-2} \log(1/\delta)$ samples to discard class $y_i$ as a mode candidate with confidence level $1-\delta$, and a sample consumes roughly $\abs{\log_2 p(y_1)}$ queries. 
This explains Theorem~\ref{thm:set-elimination}.

\begin{algorithm}
\caption{Set Elimination} 
\KwData{$\cY$, $(Y_s)\sim p$, scheduling $(n_r)\in\bN^\bN$ and $(\epsilon_r)\in\bR^\bN$, confidence interval parameter $\sigma$.}
    Set the eliminated set $S_e = \emptyset$, $r=0$, and $\cT$ a binary search tree;\\    
    \While{$\cY \backslash S_e$ is not a singleton}{
    Set $r \leftarrow r+1$;\\
    Get $n_r$ fresh samples $(Y_{j,r})_{j\in[n_r]}$;\\
    Ask $(\ind{(Y_{j,r}) \in S_e})_{j\in[n_r]}$ to remove samples that belong to the eliminated set $S_e$;\\
    We call $\hat p_r$ the 
    empirical distribution on $(Y_{j,r})_{j\in[n_r]}$ and  $\hat y_r$ the mode of its restriction to $\cY \backslash S_e$;\\
    Run Algorithm~\ref{algo:batch-search} with tree $\cT$ and parameters $\gamma = 1 / 2, \epsilon = \epsilon_r / \hat p_r(\cY \backslash S_e)$ on the samples $Y_{j,r}\not \in S_e$; \\
    This updates $\cT$ and yields a $\hat p_r(\hat y_r)/2 - \epsilon_r$-admissible partition $\cP_r$ of $\cY \backslash S_e$ with respect to $\hat p_r$;\\
    Set  $\hat p_+(S) = \hat p_r(S) + \sigma(n_r, \hat p_r)$ for $S\in\cP_r$;\\
    Set $S_e = S_e \cup \{y \in S \: | \: S\in\cP, \hat p_+(S) < \hat p_r(\hat y_r)\}$ to eliminate unlikely mode candidates;
    }
\KwResult{Return the mode estimate $\hat y_r$  of the last round $r$ as the estimated mode.}
\label{algo:set-elimination}
\end{algorithm}

\begin{theorem}
\label{thm:set-elimination}
If Algorithm~\ref{algo:set-elimination} is run with schedule $n_r := 2^r$ and $\epsilon_r := \frac{1}{4m}\paren{\frac{2}{3}}^{\frac{r}{2}}$ and confidence interval parameter $\sigma$ defined from a confidence level parameter $\delta > 0$ as\footnote{
    Note that although we do not have access to all the $\{\hat p_r(z)\}_{z\in \cY \backslash S_e}$ when running Algorithm \ref{algo:set-elimination}, we have access to their maximizer (as an output of Algorithm \ref{algo:batch-search}), which ensures that $\sigma(n_r, \hat p_r)$ is computable by the user.
}
  \begin{equation}
	\label{eq:sigma}
	\sigma(n, \hat p) = \sqrt{\frac{c\max_{z\in\cY\setminus S_e} \hat p(z)\ln(\pi^2 m n^2/\delta)}{n}}, \qquad\text{with}\qquad c = 24,
  \end{equation}
then with probability $1-\delta$ Algorithm~\ref{algo:set-elimination} terminates, identifies the right mode and consumes less than $T$ queries, where, conditionally to this successful event,
   \begin{align*}
    \frac{\E \left[ T \right]}{\log(1/\delta)} & \leq C_1\frac{p(y_1)}{\nabla_2^2}
    + C_2\sum_{i\leq m} p(y_i) \frac{p(y_1)}{\nabla_i^2} \ceil{\log_2\left(\frac{10}{p(y_1)}\right)} + o(1),
\end{align*}
for two universal constants $C_1, C_2$, $\nabla_i := p(y_1) - p(y_1)$ if $i\neq 1$ and $\nabla_1 = \nabla_2$.
\end{theorem}

Note that the asymptotic expected number of queries required only depends on the classes $y$ that are close in probability mass to $y_1$, as the contribution to the sum of all the classes $y$ that are such that $p(y) \leq p(y_1)/2$ is smaller than $4C_2\ceil{\log_2(10/p(y_1))} \log(1/\delta) / p(y_1) $. 
In particular, we have freed ourselves from any direct dependence in the number $m$ of classes. 
On the other hand, we expect to have to precisely estimate the probability mass $p(y_i)$ of the classes for which $p(y_i)$ is close to $p(y_1)$.
In line with this intuition, the dependence in those classes of the expected number of queries is roughly $C_2 \Delta_i^{-2} \ceil{\log_2(10 / p(y_1))} \log(1/\delta)$, which corresponds, up to a multiplicative factor, to the number of samples needed to disqualify $y_i$ as a mode candidate multiplied by the rough number of queries needed to precisely identify $y_i$ using a Huffman tree adapted to $p$.
These semi-heuristic arguments suggest that there should be no easy way to improve upon Algorithm~\ref{algo:set-elimination}, besides tightening the various constants in Theorem \ref{thm:set-elimination} through more careful computations.

\subsection{Proofs}

This section provides the proof for Theorem \ref{thm:set-elimination}, which is a combination of sorts of the proofs of Theorems \ref{thm:truncated} and \ref{thm:elimination}.
Let $r \geq 2$, $n_r = 2^r$, and $\hat p_r$ be the empirical probability distribution associated to the samples $(Y_{i,r})_{i\in [n_r]}$ used in the $r$-th round in Algorithm~\ref{algo:set-elimination}.
We denote by $S_r$ the random set of all classes that have been eliminated in the previous rounds, and by $\hat y_r$ the mode of $\hat p_r$ restricted to $\cY \backslash S_r$, i.e. $\hat y_r = \argmax_{y\not \in S_r } \hat p_r(y)$.
To simplify notation, we write $\sigma_r$ for $\sigma_y(n_r,\hat p_r)$, and set $\delta_r = \delta / \pi^2 m n_r^2$.

We start by introducing the same events as in the proof of Theorem \ref{thm:elimination}, with the small nuance that the distributions $\hat p_r$ for $r\in \bN$ are independent from each other, and that we only consider the subset of indices $\{n_r\}_r \subset \bN$.
The proof of the associated lemma is the same as that of Lemma \ref{lem:events_have_high_probability}.

\begin{lemma}
  \label{lem:event-set-elim}
  Let us write $\tilde{\sigma}_r = \sqrt{3 p(y_1)\ln(1/\delta_r) / n_r}$, and define the events
\begin{itemize}
    \item $A_1 = \{\hat p_r(y_1) \geq  p(y_1)/2$  for all $r$ such that  $n_r\geq 4\ln(1/\delta_r / p(y_1) \}$,
    \item $A_2 = \{\hat p_r(\hat y_r) \leq 2p(y_1)$  for all $r$ such that  $n_r\geq 4\ln(1/\delta_r) /p(y_1)\}$,
	\item $A_3 = \{ |\hat p_r(y_i) - p(y_i) | \leq  \tilde{\sigma}_r$ for all $r$ such that  $n_r\geq 4\ln(1/\delta_r)/p(y_1)$ and $i\in[m]\}$, 
    \item $A_4 = \{ \sigma_r \geq \hat p_r(\hat y_r)$ for all  $r$ such that  $n_r\leq 4 \ln(1/\delta_r)/p(y_1) \}$.
\end{itemize}
Then $\bP(\cap_{i\in[4]} \cA_i) \geq 1 - \delta$.
\end{lemma}

Once again $y_1$ cannot be eliminated when the events $(A_i)_{i\in[4]}$ hold.
\begin{lemma}
  When the events $(A_i)_{i\in[4]}$ defined in Lemma \ref{lem:event-set-elim} holds, Algorithm \ref{algo:set-elimination} does not eliminate the mode.
\end{lemma}
\begin{proof}
  Assume that $y_1$ was not eliminated at the start of round $r$.
  The round iteration defines a $(\hat p_r (\hat y_r)/2-\epsilon_r)$-admissible partition $\cP_r$ on $\cY\backslash S_r$ with respect to $\hat p_r$, and all samples $Y_{i,r}$ are identified along $\cP_r \cup \{S_r\}$.
  Necessarily $\{\hat y_r\}\in \cP_r$.
  At the end of this round, a class $y\in \cY\setminus S_r$ can only be added to the set $S_r$ of eliminated classes if it belongs to some set $S\subset \cY \backslash S_r$ such that $\hat p_r(S) + \sigma_r < \hat p_r(\hat y_r)$.
  In particular, this implies that $\hat p_r(y) + \sigma_r < p_r(\hat y_r)$, which is the criterion that was applied in the Elimination Algorithm \ref{algo:elimination} from Section \ref{sec:elim}.
  As in the proof of Theorem \ref{thm:elimination}, the events $(A_i)_{i\in[4]}$ ensure that $\hat p_r(y_1) + \sigma_r < p_r(\hat y_r)$ is never true, which means that $y_1$ cannot be eliminated.
\end{proof}

Let us now estimate the expected number of queries required for round $r$, still assuming that events $(A_i)_{i\in[4]}$ are realized.
We introduce the same event as in the proof of Theorem \ref{thm:truncated}.
\begin{equation}
    \label{eq:event-set-elim}
  \cB_r = \brace{\forall\, S\subset\cY, \quad \max\brace{|\hat p_{r-1}(S) -p(S)|,|\hat p_{r}(S) -p(S)|} \leq\frac{\epsilon_{r-1}-\epsilon_{r}}{4}}.
\end{equation}
We have already shown that
\begin{equation}
  \label{eq:bound_Brc}
  \bP({}^c\cB_r)\leq 2^{m+1} \exp\parend[\Big]{-\frac{n_{r-1}}{2}\parend[\Big]{\frac{\epsilon_{r-1}-\epsilon_r}{4}}^2} \leq 2^{m+1}  \exp\parend[\Big]{-\parend[\Big]{\frac{4}{3}}^r  \frac{C}{m^2} },
\end{equation}
for some constant $C>0$.
When $\cB_r$ holds, we can bound the number of queries similarly to what we have done for Theorem \ref{thm:truncated}.

\begin{lemma}
    Let $\cT_{r-1}$ be the tree as updated by Algorithm~\ref{algo:batch-search} at the end of round $r-1$ of Algorithm~\ref{algo:set-elimination}.
    A sample $Y_{i,r} \not \in S_r$ belonging to some set $S\in \cP_{r-1}$ only consumes at round $r$ a number of queries smaller than or equal to the depth $D(S)$ of $S$ in the tree $\cT_{r-1}$.
    When $\cB_r$ \eqref{eq:event-set-elim} holds,
    \[
        D(S)\leq 2\ceil{ \log_2\parend[\Big]{\frac{10}{p(y_1)}}}.
    \]
\end{lemma}
\begin{proof}
    For $r\geq 1$, $\cT_r$ is built as a Huffman tree on the nodes corresponding to the sets of $\cP_r$ with respect to $\hat p_r$.
    Similarly to the case of the Truncated Search Algorithm \ref{algo:truncated} and the proof of the associated Theorem \ref{thm:truncated}, when $B_r$ is realized, all sets $S\in \cP_{r-1}$ are either singletons or such that $\hat p_r(S)\leq \hat p_r(\hat y_r)/2 - \epsilon_r$.
    Hence, when Algorithm~\ref{algo:batch-search} is applied to the $n_r$ samples of round $r$, a sample $Y_{i,r} \not \in S_r$ belonging to some set $S\in \cP_{r-1}$ only consumes a number of queries smaller than or equal to the depth $D(S)$ of $S$ in the tree $\cT_{r-1}$.
    Using as in the proof of Theorem \ref{thm:truncated} the facts that $\epsilon_{r-1}<p(y_1) / 4$ and that all $S\in \cP _{r-1}$ (except at most one) satisfy
    \begin{align*}
      \hat p_{r-1}(S) & 
      \geq \frac{\hat p_{r-1}(\hat y_{r-1})/2 - \epsilon_{r-1}}{2}
      > \frac{p( y_{1}) - \frac{\epsilon_{r-1} -\epsilon_r}{4} - 2\epsilon_{r-1}}{4} 
       >  \frac{p( y_{1}) - \frac{9\epsilon_{r-1}}{4}}{4} > \frac{7 p(y_{1})}{64}.
    \end{align*}
    We conclude that $D(S)\leq 2\ceil{\log_2(64 / 7p(y_1))}$ thanks to Lemma \ref{lem:depth_in_Huffman_tree}.
\end{proof}

If the event $\cB_r$ is not realized, we can roughly upper bound the number of queries spent per sample by $m$, hence we can upper bound the total number of queries by $n_r m$.
In the special case $r=1$, we similarly need at most $n_r m$ queries.

Let $T_r$ be the total number of queries needed for round $r$ (it is a random variable).
We spend $n_r$ queries at the start of the round to check whether $Y_{i,r} \in S_r$ for each sample $Y_{i,r}$, $i\in [n_r]$.
We have shown that if $r\geq 2$, then
\begin{equation}
    \label{eq:set-elimination-first-bound-T}
    T_r \leq  \sum_{i=1}^{n_r} 1 +  1_{\{Y_{i,r}\not \in S_r\}} \left( 2\ceil{\log_2\left(\frac{10}{p(y_1)}\right)} + 1_{{}^c B_r}m \right),
\end{equation}
while if $r=1$, $T_r \leq n_1 m$.

As previously, we will bound (the expectation of) $1_{\{Y_{i,r}\not \in S_r\}}$ by a deterministic quantity under the events $(A_j)_{j\in[4]}$.

\begin{lemma}
    \label{lem:bound-nry}
    When the events $(A_j)_{j\in[4]}$ holds, then $y_i\neq y_1$ is added to $S_r$ no later than when $r=r(y)$, where the elimination time $r(y)$ is defined as
    \[
    r(y) = \min \left\{ r  \midvert  n_{r} > 108\max\left((p(y_1)-p(y))^{-2}, (p(y_1)/2)^{-2} \right)  p(y_1) \ln(1/\delta_r) \right\} 
    \]
    In particular
    \[
        n_{r(y)}\leq 216\max\left((p(y_1)-p(y))^{-2}, (p(y_1)/2)^{-2} \right)  p(y_1) \ln(1/\delta_r).
    \]
\end{lemma}
\begin{proof}
    We write $\tilde{\sigma}_{r} = \sqrt{3 p(y_1)\ln(1/\delta_{r}) / n_{r}}$; we have seen in the proof of Theorem \ref{thm:elimination} that $A_2$ implies $4\tilde{\sigma}_{r} \geq \sigma_{r} $.
    Let us now consider $y\in \cY$, and let $S_{r}(y)$ be the set of the partition $\cP_{r}$ to which $y$ belongs, assuming that it has not yet been eliminated at the start of round $r$. 
    Note that by definition of $\cP_{r}$, either $\hat p_{r}(S_{r}(y)) \leq \hat p_{r}(\hat y_{r})/2$ or $S_{r}(y) = \{y\}$.
    
    We first examine the case $p(y)\leq p(y_1)/2$. The event $A_3$ implies that $\hat p_{r}(y)\leq p(y_1)/2 + \tilde{\sigma}_{r}$ and that $\hat p_{r}(\hat y_{r})/2 \leq p(\hat y_{r})/2 + \tilde{\sigma}_{r}/2  \leq p(y_1)/2 + \tilde{\sigma}_{r}/2$.
    Thus $\hat p_{r}(S_{r}(y)) \leq p(y_1)/2 + \tilde{\sigma}_{r}$.
    Thanks to $A_3$ again, we know that $\hat p_{r}(\hat y_{r})\geq p(y_1) - \tilde{\sigma}_{r}$.
    Hence 
    \[\hat p_{r}(\hat y_{r}) - \hat p_{r}(S_{r}(y)) \geq p(y_1)/2 - 2\tilde{\sigma}_{r}. \]
    If $p(y_1)/2 \geq 6 \tilde{\sigma}_{r}$, then
    \[\hat p_{r}(\hat y_{r}) - \hat p_{r}(S_{r}(y)) \geq 4\tilde{\sigma}_{r} \geq \sigma_{r} \]
    and $S_{r}(y)$ gets eliminated at the end of round $r$.
    
    Now let us consider the case $p(y)\geq p(y_1)/2$.
    Similarly, the event $A_3$ ensures that $\hat p_{r}(y)\leq p(y) +\tilde{\sigma}_{r}$ and that 
    \[\hat p_{r}(S_{r}(y)) \leq \hat p_{r}(\hat y_{r})/2  \leq p(y_1)/2 + \tilde{\sigma}_{r}/2 \leq p(y) + \tilde{\sigma}_{r}/2 \]
    in the case where $|S_{r}(y)| \geq 2$, hence that
    \[\hat p_{r}(S_{r}(y)) \leq  p(y) + \tilde{\sigma}_{r}. \]
    As above, $\hat p_{r}(\hat y_{r})\geq p(y_1) - \tilde{\sigma}_{r}$, hence 
    \[\hat p_{r}(\hat y_{r}) - \hat p_{r}(S_{r}(y)) \geq p(y_1) -p(y) - 2\tilde{\sigma}_{r}. \]
    If $p(y_1)-p(y) \geq 6 \tilde{\sigma}_{r}$, then
    \[\hat p_{r}(\hat y_{r}) - \hat p_{r}(S_{r}(y)) \geq 4\tilde{\sigma}_{r} \geq \sigma_{r}, \]
    and $S_{r}(y)$ must be eliminated at the end of round $r$.
    
    We have shown that if $p(y)\leq p(y_1)/2$,  $y$ is eliminated at the latest at the end of the first round $r$ for which $p(y_1)/2 \geq 6 \tilde{\sigma}_{r}$, and if $p(y) \geq p(y_1)/2$, it is eliminated at the latest at the end of the first round for which $p(y_1)-p(y) \geq 6 \tilde{\sigma}_{r}$.
    Note that as in the proof of Theorem \ref{thm:elimination}, the condition 
    $\rho > 6 \tilde{\sigma}_{r}$ (for $\rho \in \{p(y_1)-p(y), p(y_1)/2 \}$) is equivalent to 
    \[n_{r} >  108\rho^{-2}  p(y_1) \ln(1/\delta_r),\]
    which defines the elimination time $r(y)$.
\end{proof}

The previous lemma allows us to bound $\ind{Y_{i,r}\notin S_r}$ by $\ind{r \leq r(y)}$.
We are now left with the computation of the upper bound in Equation \eqref{eq:set-elimination-first-bound-T}.
We provide the derivations in Appendix \ref{app:set-elim}, which yield the following lemma.

\begin{restatable}{lemma}{setelimquery}
  With the events $(A_i)_{i\in[4]}$ as defined in Lemma \ref{lem:event-set-elim},
 \begin{align*}
    &\frac{\E \bracket{ T \midvert (A_i)_{i\in[4]}}}{\ln(1/\delta)} \leq 432\frac{p(y_1)}{(p(y_1)-p(y_2))^2}  p(y_1)
    \\& + 864\sum_{y\in\cY \text{ s.t. } p(y)< p(y_1)/2}  p(y)  \frac{4}{p(y_1)^2} p(y_1) \ceil{\log_2\left(\frac{10}{p(y_1)}\right)}
    \\& + 864\sum_{y\in\cY \text{ s.t. } p(y)\geq p(y_1)/2}  p(y) \frac{1}{(p(y_1)-p(y))^2} p(y_1)\ceil{\log_2\left(\frac{10}{p(y_1)}\right)} +o(1).
  \end{align*}
\end{restatable}

Note that since \[\max\left((p(y_1)-p(y))^{-2}, (p(y_1)/2)^{-2} \right) \leq  4 (p(y_1)-p(y))^{-2}\]
for any $y\in \cY$, we can weaken and simplify this upper bound as
\[
    \frac{\E \bracket{T \midvert (A_i)_{i\in[4]}}}{\ln(1/\delta)} \leq 432\frac{p(y_1)^2}{(p(y_1)-p(y_2))^2}
    + 3456\sum_{y\in\cY } \frac{p(y_i)p(y_1)}{\nabla_i^2} \ceil{\log_2\left(\frac{10}{p(y_1)}\right)} +o(1)).
\]
This completes the proof.

\subsection*{Aside on ``Forever Running'' Algorithms}

It is easy to reuse parts of the proofs of Theorems \ref{thm:elimination} and \ref{thm:set-elimination} to turn algorithms that output the empirical mode, namely the Exhaustive, Adaptive and Truncated Search Algorithms, into $(\epsilon,\delta)$-PAC algorithms for $\epsilon \geq 0$ and $\delta > 0$.
On can also modify the Elimination and Set Elimination Algorithms to turn them into algorithms that run forever for ever increasing accuracy by adapting the doubling trick credited to \citet{Auer1995} to our partial feedback setting. The asymptotic expected number of queries as a function of the probability of error $\delta$ would scale in the same way as for the initial algorithms, though with worse constants.


\section*{Conclusion}
This article introduces Problem \ref{Problem}, a new framework in which to formalize the problem of active learning with weak supervision. 
It presents three important ideas on how to solve it, namely the use of adaptive entropy coding, coarse sufficient statistics and confidence intervals, and illustrates these ideas through increasingly complex algorithms.
Finally, it combines those ideas into Algorithm~\ref{algo:set-elimination}, which provably only needs an expected number
\begin{align*}
     \E \left[ T \right]  & \leq  \parend[\Bigg]{  C_1\Delta_2^{-2}
    + C_2\sum_{i\leq m}\Delta_i^{-2}\ceil{\log_2\parend[\Big]{\frac{10}{p(y_1)}}}+ o(1) } \log(1/\delta) ,
\end{align*}
of queries to identify the mode with probability at least $1-\delta$ (assuming that $p(y_1)$ is bounded away from $1$). 


\begin{appendix}
\section*{}

\subsection{Information Projection Computation}
\label{app:mode}


Recall the definitions of the set 
\[\cQ = \brace{q \in \prob\cY \midvert \argmax q(y) \neq \argmax p(y)},\]
\begin{equation}
    \tag{\ref{eq:Qn-}}
    \cQ_{n,-} = \brace{q \in \prob\cY \cap n^{-1}\cdot\bN^\cY \midvert y_1 \notin \argmax q(y)},
\end{equation}
and
\begin{equation}
    \tag{\ref{eq:Qn+}}
    \cQ_{n,+} = \brace{q \in \prob\cY \cap n^{-1}\cdot\bN^\cY \midvert \argmax q(y) \neq \argmax p(y)}.
\end{equation}
In this section, we prove Lemma \ref{lem:kl-proj-short} from Section \ref{sec:mode}, which we restate here for the reader's convenience, as well as some related results.

\infoprojection*

Lemma \ref{lem:kl-proj-short} is a substatement of Lemmas \ref{lem:information_projection_main_statement} and \ref{lem:information_projection_Qn} below.
We first prove the following intermediate result.
\begin{lemma}\label{lem:intermediate_result_divergence}
  Let $\cY = \{y_1,\ldots,y_m\}$.
  Let $q,p\in \bP (\cY)$ and $y',y'' \in \cY$ be such that
  \[
    \frac{q(y')}{p(y')} > \frac{q(y'')}{p(y'')}.
  \]
  Let us define for $\epsilon \in (0, \min\{q(y'), 1-q(y'')\}]$, the distribution $q_\epsilon \in \prob\cY$
  \[ 
    q_\epsilon(y) = \left\{ \begin{array}{ll}
        q(y) - \epsilon & \text{if } y = y'  \\
        q(y) + \epsilon & \text{if } y = y''  \\
        q(y) & \text{otherwise,}
    \end{array} \right.
  \]
  then
  \[
    \epsilon \in \left(0, \frac{q(y')p(y'') - q(y'')p(y')}{p(y')+q(y')}\right] \qquad \Rightarrow \qquad D(q_\epsilon\|p) < D(q\|p).
  \]
\end{lemma}
\begin{proof}
    Indeed, 
    \[
        D(q_\epsilon\|p) = (q(y')-\epsilon)\log\paren{\frac{q(y')-\epsilon}{p(y')}}+ (q(y'')+ \epsilon)\log\paren{\frac{q(y'')+\epsilon}{p(y'')}} + C(q,p),
    \]
    where $C(q,p)$ are some terms that do not depend on $\epsilon$, and 
    \[
        \frac{\diff}{\diff \epsilon}D(q_\epsilon\|p) = \log\left( \frac{p(y')(q(y'') +\epsilon)}{(q(y')-\epsilon)p(y'')} \right)
    \]
    is strictly negative for $\epsilon >0$ such that $\frac{p(y')(q(y'') +\epsilon)}{(q(y')-\epsilon)p(y'')}<1$, which leads to the condition on $\epsilon$.
\end{proof}

We can now characterize the information projection $\min_{q\in Q}D(q\|p)$.

\begin{lemma}\label{lem:information_projection_main_statement}
  Let $\cY = \{y_1,\ldots,y_m\}$ and $p\in \bP (\cY)$ be such that $p(y_1)>p(y_2) > 0$ and $p(y_{i-1})\geq p(y_i)$ for all $i=2,\ldots,m$.
  Define 
  \[
     \cQ := \brace{q\in \prob\cY \midvert \exists\, y\in \cY \setminus \{y_1\} \text{ s.t. } q(y) \geq q(y_1) }.
  \]
  There exists $q^*\in \cQ$ such that  $D(q^*\|p) = \min_{q\in Q}D(q\|p)$, and\footnote{This characterization of $q^*$ is exact up to some renaming of classes $y$ with equal probability masses.}
  \[
    q^*(y_1) = q^*(y_2) = \frac{1- \lambda(1-p(y_1)-p(y_2))}{2}
  \]
  and 
  \[
    q^*(y_i) = \lambda p(y_i) \quad \forall i \in \{3,\ldots,m\}
  \]
  for 
  \[\lambda = \frac{1}{1 - (\sqrt{p(y_1)}-\sqrt{p(y_2)})^2}.\]
    It satisfies \[
    D(q^*\| p) = - \log( 1 - (\sqrt{p(y_1)}-\sqrt{p(y_2)})^2).
  \]
\end{lemma}
\begin{proof}
    The statement is trivial when $m=2$ or $\sum_{i=3}^m p(y_i) = 0$  (with $\lambda = 1$), which we assume henceforth not to be the case.

    Observe first that if $y\in \cY$ is such that $p(y) = 0$, then necessarily $q^*(y)=0$.
    Indeed, if $p(y) = 0$ and $q^*(y)>0$, then $D(q^*\|p) = \infty$, while $ \min_{q\in Q}D(q\|p) < 0$ (consider for example $q\in \bP(\cY)$ defined as $q(y_1) = q(y_2) = (p(y_1)+p(y_2)) / 2$ and $q(y) = p(y)$ for all $y\in \cY\backslash \{y_1,y_*\}$). 

    Let $y_* \in \cY\setminus\brace{y_1}$ be a mode of $q^*$, which exists by hypothesis.
     We must have
    \begin{equation*}
        \label{eq:tmp-proj}
        \forall\, y', y'' \in \cY\setminus\{y_*, y_1\} \quad \text{s.t. } p(y')\neq 0, p(y'')\neq 0, \qquad \frac{q(y')}{p(y')} \leq \frac{q(y'')}{p(y'')},
    \end{equation*}
    as otherwise one could apply Lemma \ref{lem:intermediate_result_divergence} to $y'$ and $y''$ to find $q_\epsilon \in \cQ$ such that $q_\epsilon(y_1) = q(y_1) \leq q(y_*) = q_\epsilon$, hence $q_\epsilon\in\cQ$, and $D(q_\epsilon\| p) < D(q\| p)$, which would contradict the definition of $q^*$.
    
    In particular, this implies by symmetry that
    \[
        \exists\,\lambda \geq 0\quad \text{s.t.} \quad \forall\,y\in\cY\setminus \brace{y_1, y_*} \quad q(y) = \lambda p(y).
    \]

        We observe that
    \[
        q^*(y_1) = q^*(y_*).
    \]
    Indeed, we always have $q^*(y_*) / p(y_*) > q^*(y_1) / p(y_1)$, since $q^*(y_*) \geq  q^*(y_1)$ and $p(y_1)  >p(y_*)$. If $q^*(y_1) < q^*(y_*)$, we could apply Lemma \ref{lem:intermediate_result_divergence} to $y_1, y_*$ with a small enough $\epsilon$ to find $q_\epsilon \in Q$ with $D(q_\epsilon\| p) < D(q\| p)$.

    Now let us show that
    \[
        \exists\, y \in \{y\in \cY\setminus \{y_1\} | q^*(y) = q^*(y_*)\} \quad\text{ s.t.} \quad p(y) = p(y_2).
    \]
    Indeed, assume that it is not the case: then $y_*$ does not satisfy those conditions, hence it must be such that  $p(y_*)<p(y_2)$, and $y_2$ does not satisfy those conditions, hence it must be such that $q^*(y_*) > q^*(y_2) $.
    Consider the distribution $q$ defined as $q(y_2) = q^*(y_*)$, $q(y_*) = q^*(y_2)$ and $q(y) = q^*(y)$ for all $y\in \cY \setminus\{y_2,y_*\}$. Then $q\in Q$, and
    \begin{align*}
     D(q^*\|p) - D(q\|p) & = q^*(y_*)\log(q^*(y_*)/p(y_*)) + q^*(y_2) \log(q^*(y_2)/p(y_2) )
     \\&\qquad -q^*(y_2)\log(q^*(y_2)/p(y_*)) - q^*(y_*) \log(q^*(y_*)/p(y_2))
     \\& = q^*(y_*)\log(p(y_2)/p(y_*)) + q^*(y_2) \log(p(y_*)/p(y_2) )
     \\&= (q^*(y_*) - q^*(y_2) )\log(p(y_2)/p(y_*)) >0, 
    \end{align*}
    leading to a contradiction.
    Hence there exists $\tilde{y} \in \{y\in \cY\setminus \{y_1\} | q^*(y) = q^*(y_*)\}$ such that $p(\tilde{y}) = p(y_2)$.
    Consider now the distribution $q$ defined as $q(\tilde{y}) = q^*(y_2)$, $q(y_2) = q^*(\tilde{y})$ and $q(y) = q^*(y)$ for all $y\in\cY\backslash\{y^*, \tilde{y}\}$; then $q\in \cQ$ and $D(q\|p) = D(q^*\|p)$, and without loss of generality we can rename $q$ into $q^*$ for the remainder of the proof.
   
    Let us write $x =  q^*(y_1) = q^*(y_2)$.
    As
    \[1 = \sum_{i=1 }^m q^*(y_i) = 2 x +  \sum_{i=3 }^m \lambda p(y_i) = 2x + \lambda(1-p(y_1)-p(y_2)),\]
    we find that $x = (1- \lambda(1-p(y_1)-p(y_2))) / 2$.

    We should have 
    \[
        \frac{q^*(y_1)}{p(y_1)} \leq \lambda = \frac{q^*(y)}{p(y)} \leq \frac{q^*(y_2)}{p(y_2)}
    \]
    for $y\in\cY\setminus\{y_1,y_2\}$, as we could otherwise apply Lemma \ref{lem:intermediate_result_divergence} to reach a contradiction with  the definition of $q^*$.
      This leads to the constraint
    \[
        \frac{x}{p(y_1)} = \frac{1- \lambda(1-p(y_1)-p(y_2))}{2p(y_1)}  \leq \lambda,\text{ i.e. }
        \lambda \geq \frac{1}{(1+p(y_1)-p(y_2))},
    \]
    as well as
    \[
        \lambda \leq \frac{x}{p(y_2)} = \frac{1- \lambda(1-p(y_1)-p(y_2))}{2p(y_2)}, \text{ i.e. }
        \lambda \leq \frac{1}{(1-p(y_1)+p(y_2))}.
    \]
    
    Hence we have shown that $q^*$ must be some distribution $q_\lambda$ defined as $q_\lambda(y_i) = \lambda p(y_i)$  for all $i\geq 3$ and
    \[
        q_\lambda(y_1) = q_\lambda(y_2) = \frac{1- \lambda(1-p(y_1)-p(y_2))}{2}.
    \]
    for some
    \[
    \lambda \in \bracket{\frac{1}{1+p(y_1)-p(y_2)},\frac{1}{1-p(y_1)+p(y_2)}}.
  \]
    If we show that $D(q_\lambda\|p)$ is minimized for $\lambda = \frac{1}{1 - (\sqrt{p(y_1)}-\sqrt{p(y_2)})^2}$ and that $D(q^*\|p) = - \log( 1 - (\sqrt{p(y_1)}-\sqrt{p(y_2)})^2)$, the proof is complete; this is the statement of Lemma \ref{lem:information_projection_lambda_optoptimization} below.
\end{proof}

\begin{lemma}\label{lem:information_projection_lambda_optoptimization}
    Consider the family of distributions $q_\lambda$ defined for $\lambda \in [0,(1-p(y_1)-p(y_2))^{-1}]$ as $q_\lambda(y_i) = \lambda p(y_i)$  for all $i\geq 3$ and
    \[
        q_\lambda(y_1) = q_\lambda(y_2) = \frac{1- \lambda(1-p(y_1)-p(y_2))}{2}.
    \]
    Then 
    \[\argmin_\lambda D(q_\lambda\|p) = \frac{1}{1 - (\sqrt{p(y_1)}-\sqrt{p(y_2)})^2}\]
    and 
    \[\min_\lambda D(q_\lambda\|p) = \- \log( 1 - (\sqrt{p(y_1)}-\sqrt{p(y_2)})^2).\]

\end{lemma}
\begin{proof}
    We compute
    \begin{align*}
    D(q_\lambda\|p)  &= \frac{1- \lambda(1-p(y_1)-p(y_2))}{2}
    \log \left(\frac{(1- \lambda(1-p(y_1)-p(y_2)))^2}{4p(y_1)p(y_2)}\right) 
    \\&\qquad\qquad+ \sum_{i=3}^m \lambda p(y_i)\log(\lambda) 
    \\&=  \frac{1- a\lambda}{2}\log \left(\frac{(1- a\lambda)^2}{b}\right)  + a\lambda \log(\lambda),
    \end{align*}
    where $a = 1 - p(y_1) - p(y_2)$ and $b = 4p(y_1)p(y_2)$, which implies that the derivative reads
    \begin{align*}
    \frac{\diff}{\diff\lambda}D(q_\lambda\|p) 
    &= -\frac{a}{2}\log\paren{\frac{(1-a\lambda)^2}{b}} + (1-a\lambda)\cdot \frac{-a}{1 - a\lambda} + a\log(\lambda) + a 
    \\&= \frac{a}{2} \log\paren{\frac{b\lambda^2}{(1-a\lambda)^2}} 
    \\& = \frac{1-p(y_1)-p(y_2)}{2}\log\left(\frac{4p(y_1)p(y_2)\lambda^2}{(1-\lambda(1-p(y_1)-p(y_2)))^2} \right).
    \end{align*}
    We can study the sign of this derivative,
    \begin{align*}
        \sign\paren{\frac{\diff}{\diff\lambda}D(q_\lambda\|p)} 
       &= \sign\paren{\frac{b\lambda^2}{(1-a\lambda)^2} -1}
        = \sign\paren{b\lambda^2 - (1-a\lambda)^2}
     \\&= \sign\paren{(b - a^2)\lambda^2 + 2a\lambda - 1}.
    \end{align*}
    The roots of this polynomial are
    \begin{align*}
     \lambda_\pm
     &=  \frac{-a \pm \sqrt{a^2 + b - a^2}}{b - a^2}  
     = \frac{-a \pm \sqrt{b}}{b - a^2} 
     = \frac{1}{a \pm \sqrt{b}}
     \\&= \frac{1}{1 - p(y_1) - p(y_2) \pm 2\sqrt{p(y_1)p(y_2)}}
     = \frac{1}{1 - (\sqrt{p(y_1)} \mp \sqrt{p(y_2)})^2}.
    \end{align*}
    The sign of the leading coefficient of the polynomial is
    \begin{align*}
        \sign(b - a^2) &= \sign(\sqrt{b} - a) = \sign(2\sqrt{p(y_1)p(y_2)} + p(y_1) + p(y_2) - 1)
        \\&= \sign((\sqrt{p(y_1)} + \sqrt{p(y_2)})^2 - 1)
        = \sign(\sqrt{p(y_1)} + \sqrt{p(y_2)} - 1).
    \end{align*}
    As a consequence, there are two possibilities.
    
    If $\sqrt{p(y_1)} + \sqrt{p(y_2)} \geq 1$, then $b - a^2$, as well as $\sqrt{b} - a$, is positive, and $\diff D / \diff \lambda$ is negative from $\lambda_- = (a - \sqrt{b})^{-1} \leq 0$ to $\lambda_+ = (a + \sqrt{b})^{-1}$ and positive afterwards, in which case $\lambda_+$ is the minimizer of $D(q_\lambda \| p)$.
    
    If $\sqrt{p(y_1)} + \sqrt{p(y_2)} \leq 1$, then $b - a^2$, as well as $\sqrt{b} - a$ is negative, and $\diff D / \diff \lambda$ is negative from $\lambda=0$ to $\lambda_+ = (a + \sqrt{b})^{-1} $, then positive until $\lambda_- = (a - \sqrt{b})^{-1}$ and negative afterwards, in which case either $\lambda_+$ or the right extremity of the domain of $\lambda$ is the minimizer of $D(q_\lambda \| p)$ on the domain of $\lambda$.
    Yet this right extremity is $(1 - p(y_1) - p(y_2))^{-1}$ is smaller than $\lambda_-$, since
    \[
        p(y_1) + p(y_2) \leq p(y_1) + 2\sqrt{p(y_1)p(y_2)} + p(y_2) = (\sqrt{p(y_1)} + \sqrt{p(y_2)})^2,
    \]
    implies
    \[
        (1 - p(y_1) - p(y_2))^{-1} \leq (1 - (\sqrt{p(y_1)} + \sqrt{p(y_2)})^2)^{-1} = \lambda_-,
    \]
    hence the minimizer has to be $\lambda_+$.

    We conclude that $\lambda_+$ is always the minimizer of $D(q_{\lambda_+}\| p)$.
    Hence, using that $\lambda_+$ solves $b\lambda^2 = (1-a\lambda)^2$,
    \begin{align*}
    D(q_{\lambda_+}\|p) 
    &=  \frac{1- a\lambda_+}{2}\log\paren{\frac{(1- a\lambda_+)^2}{b}}  + a\lambda_+ \log(\lambda_+)
    \\&=  \frac{1- a\lambda_+}{2}\log \paren{\lambda_+^2}  + a\lambda_+ \log(\lambda_+)
    =  \log(\lambda_+).
    \end{align*}
   This ends the computation of the information projection.
\end{proof}
To close this subsection, let us finally prove the upper bound on $\min_{q\in \cQ_{n, -}} D(q \| p)$ from Lemma \ref{lem:kl-proj-short} (in fact, we prove a little more).
\begin{lemma}\label{lem:information_projection_Qn}
Let $\cY = \{y_1,\ldots,y_m\}$, $n\geq 1$,
\[
    \cQ_{n,-} = \brace{q \in \prob\cY \cap n^{-1}\cdot\bN^\cY \midvert y_1 \notin \argmax q(y)}
\]
and 
\[\cQ = \brace{q \in \prob\cY \midvert \argmax q(y) \neq \argmax p(y)}.\]

Then 
\begin{align*}
 &\min_{q\in \cQ_{n, -}} D(q \| p) \leq \min_{q\in \cQ} D(q \| p)
 \\&\quad +  \frac{m}{n} \Biggl(\max\left\{ -\log(\lambda), \log\left(\frac{1- \lambda(1-p(y_1)-p(y_2))}{2p(y_2)}\right)\right\} +1 \Biggl)
 \\&\quad + \frac{1}{n} \sum_{y\in \cY_{>0} }  \log\left(1+\frac{1}{nq^*(y)} 
\right)
\end{align*}
where $\lambda = (1 - (\sqrt{p(y_1)} - \sqrt{p(y_2)})^2)^{-1}$ and $\cY_{>0} = \{y\in \cY \;| p(y) \neq 0 \}$.
\end{lemma}

\begin{proof}
    Let $q^*\in \argmin_{q\in \cQ} D(q \| p)$ be such that $q^*(y_1) = q^*(y_2) = \frac{1- \lambda(1-p(y_1)-p(y_2))}{2}$ and $q^*(y) =p(y)$ for $\lambda = (1 - (\sqrt{p(y_1)} - \sqrt{p(y_2)})^2)^{-1}$ (as in Lemma \ref{lem:information_projection_main_statement}).
Then for $n> m$, one can construct  $q_n\in \cQ_{n, +}$ such that $q_n(y)  \leq q^*(y) + 1/n$ for all $y\in \cY$ and such that $q_n(y) = 0$ if $y\not \in \cY_{>0} $.
The general idea is as follows: let  $q_n(y) = 0$ for all $y\not \in \cY_{>0} $. 
Then let $q_n(y_1) = \left \lfloor q^*(y_1)n \right \rfloor /n$ and $q_n(y_2) = \left(\left \lfloor q^*(y_1)n\right \rfloor +1\right) /n$, and $q_n(y) = \left (\left \lfloor q^*(y)n \right \rfloor + \epsilon_{y,n}\right) /n$ for $y\in \cY\backslash\{y_1,y_2\}$ such that $q^*(y)\neq 0$,  with $\epsilon_{y,n} \in \{0,1\}$ chosen so that $\sum_{y\in\cY} q_n(y) = 1$ (small adjustments can be necessary depending on whether $q^*(y_1)n \in \bN$ etc.).

Then 
 \[D(q_n \| p) = \sum_{y\in \cY} q_n(y)\log\left(\frac{q_n(y)}{p(y)}\right) = \sum_{y\in \cY}  q_n(y) \log\left(\frac{q^*(y) }{p(y)} \right) + \sum_{y\in \cY}  q_n(y) \log\left( \frac{q_n(y)}{q^*(y)}\right).\]
The first sum can be bounded as 
\begin{align*} 
\sum_{y\in \cY}  q_n(y) \log\left(\frac{q^*(y) }{p(y)} \right)
&= \sum_{y\in \cY}  q^*(y) \log\left(\frac{q^*(y) }{p(y)} \right)  + \sum_{y\in \cY} ( q_n(y)- q^*(y)) \log\left(\frac{q^*(y) }{p(y)} \right) 
\\&\leq D(q^*\|p) + \frac{m}{n} \max\left\{ \log(\lambda), \log\left(\frac{q^*(y_2)}{p(y_2)}\right)\right\},
\end{align*}
and the second sum can be bounded as follows \begin{align*} 
\sum_{y\in \cY}  q_n(y) \log\left( \frac{q_n(y)}{q^*(y)} \right)
&  \leq \sum_{y\in \cY_{>0} }   (q^*(y) + 1/n) \log\left(1+\frac{1}{nq^*(y)}  \right)
\\& \leq \sum_{y\in \cY_{>0} }   q^*(y) \log\left(1+\frac{1}{nq^*(y)}  \right) + \sum_{y\in \cY_{>0} }  \frac{1}{n} \log\left(1+\frac{1}{nq^*(y)}  \right) 
\\& \leq \sum_{y\in \cY_{>0} }   q^*(y) \frac{1}{nq^*(y)}  + \sum_{y\in \cY_{>0} }  \frac{1}{n} \log\left(1+\frac{1}{nq^*(y)} 
\right) 
\\& \leq\frac{m}{n} + \frac{1}{n}  \sum_{y\in \cY_{>0} }  \log\left(1+\frac{1}{nq^*(y)} 
\right) .
\end{align*}
Thus 
\begin{align*}
 &\min_{q\in \cQ_{n, -}} D(q \| p)  \leq D(q_n \| p) 
  \\ \leq& D(q^* \| p)  +  \frac{m}{n} \left( \max\left\{ \log(\lambda), \log\left(\frac{q^*(y_2)}{p(y_2)}\right)\right\}+1 \right)+ \frac{1}{n}\sum_{y\in \cY_{>0} }  \log\left(1+\frac{1}{nq^*(y)} 
\right).
\end{align*}
This concludes the proof of the Lemma.
\end{proof}
An alternative construction yields 
\[
    \min_{q\in \cQ_{n, -}} D(q \| p)\leq D(q^* \| p) + C(m,p(y_1),p(y_2)) / \sqrt{n}
\]
for some constant $C(m,p(y_1),p(y_2))$ that only depends on $m,p(y_1),p(y_2)$ (rather than on all $p(y)$), though at the cost of a less favorable asymptotic behaviour in $n$.

\subsubsection*{Sanity Check with Cram\'er-Chernoff method}
We remark that the upper bound of Theorem \ref{thm:empirical-mode} can be proved more directly.
Consider first the simple union bound
\begin{align*}
    \bP(\hat y_n \neq y_1) 
    &= \bP\parend[\big]{\min_{i\neq 1}\sum_{j\in[n]} \ind{Y_j = y_1} - \ind{Y_j = y_i} \leq 0}
    = \bP\parend[\big]{\cup_{i\neq 1}\braced[\big]{\sum_{j\in[n]} \ind{Y_j = y_1} - \ind{Y_j = y_i} \leq 0}}
    \\&\leq \sum_{i\neq 1}\bP\parend[\big]{\sum_{j\in[n]} \ind{Y_j = y_1} - \ind{Y_j = y_i} \leq 0}
    \leq (m - 1) \bP\parend[\big]{\sum_{j\in[n]} \ind{Y_j = y_1} - \ind{Y_j = y_2} \leq 0}.
\end{align*}

The term $\bP\parend[\big]{\sum_{j\in[n]} \ind{Y_j = y_1} - \ind{Y_j = y_2} \leq 0}$ can be bounded with Chernoff's method:
\begin{lemma}[Chernoff's bound]
    \label{lem:Chernoff}
    Let $p \in \prob{\cY}$ be a distribution over $\cY$, and $(Y_j)$ be $n$ independent samples distributed according to $p$ with $p(y_1) > p(y_i)$. Then
    \[
        \bP(\sum_{j\in[n]} \ind{Y_j=y_1} - \ind{Y_j=y_i} \leq 0) \leq \exp\paren{- n\Delta_i^2},
    \]
    where $\Delta_i$ is defined in Equation \eqref{eq:delta-*}.
\end{lemma}
\begin{proof}
    This can be proven using the fact that, for a random variable $X$,
    \[
        \bP(X \geq 0) = \inf_{t > 0}\bP(\exp(tX) \geq 1) \leq \inf_{t>0}\E[e^{tX}].
    \]
    As a consequence, with $X = \sum X_j$ and $X_j = \ind{Y_j = y_i} - \ind{Y_j = y_1}$,
    \[
        \bP(X \geq 0) \leq \inf_{t>0} \E[e^{tX}] = \inf_{t>0} \E[e^{t X_1}]^n = \inf_{t>0}\exp(n \ln(\E[e^{tX_1}])).
    \]
    We are left with a simple computation,
    \begin{align*}
        \inf_{t>0}\E[e^{tX_1}] 
        &= \inf_{t>0}\paren{e^t p(y_i) + (1-p(y_i)-p(y_1)) + e^{-t} p(y_1)}
        \\&= \paren{2\sqrt{p(y_1)p(y_i)} + (1 - p(y_1) - p(y_i))}
        = \paren{1 - (\sqrt{p(y_1)} - \sqrt{p(y_i)})^2}
   \end{align*}
   where we used that the infimum is found for $e^t = \sqrt{p(y_1) / p(y_i)} > 1$.
\end{proof}
In fact, we know from \citet{Cramer1938} that Chernoff's bound is asymptotically exponentially tight.
As we also have
\begin{align*}
    \bP(\hat y_n \neq y_1) 
    \geq \bP\parend[\big]{\sum_{j\in[n]} \ind{Y_j = y_1} - \ind{Y_j = y_2} \leq 0},
\end{align*}
we get that
\[
    \lim \frac{\ln(\bP(\hat y_n \neq y_1))}{n} = \lim \frac{\ln(\bP\parend[\big]{\sum_{j\in[n]} \ind{Y_j = y_1} - \ind{Y_j = y_2} \leq 0)}}{n} = \Delta_2^2.
\]
As mentioned in Section \ref{sec:mode}, \citet{Dinwoodie1992} states that $\ln(\bP(\hat y_n \neq y_1)) / n$ is always below $- \min_{q\in\cQ} D(q\| p)$, while \citet{Sanov1957} states that $\ln(\bP(\hat y_n \neq y_1)) / n$ converges to $- \min_{q\in\cQ} D(q\| p)$ as $n$ grows large when $\cY$ is discrete.
This shows without any further computations that $\min_{q\in\cQ} D(q\| p) = \Delta_2^2$, and implies the upper bound of Theorem \ref{thm:empirical-mode}. 

\subsection{Additional Proofs for Section \ref{sec:exhaustive}}
\label{app:exhaustive}

\huffmandepth*
\begin{proof}
Let us proceed by induction on $\depth_\cT(V)$.
 If $\depth_\cT(V) =0$, i.e. $V=R$, or if $\depth_\cT(V) =1$, the statement is trivial.  
 Let us assume that $\depth_\cT(V) \geq 2$.
 At some point during the construction of $\cT$, the element $V$ was merged with another element $V'$ to create a new parent node $P$. At that point, $V$ and $V'$ were the two elements with smallest value $v$ in the heap $\cS$. If $v(V') \geq v(V)$, then $v(P)\geq 2v(V)$. By induction, 
 \begin{align*}
    \depth_\cT(V) & = \depth_\cT(P) +1 \leq 2 \ceil{ \log_2(1/v(P))} +1  \leq  2 \ceil{ \log_2(1/2v(V))} +1 
    \\& \leq  2 \ceil{ \log_2(1/v(V)) - 1} +1  \leq  2 \ceil{ \log_2(1/v(V))}.
 \end{align*}
If $v(V') \leq v(V)$, then any $v(P)\geq v(V)$, and as any other element $V''$ must be such that $v(V'')\geq v(V)$, the parent $\tilde{P}$ (which results from the merging of $P$ and some $V'')$ of $P$ satisfies $v(\tilde{P}) \geq 2v(V)$.
By induction, 
\begin{align*}
    \depth_\cT(V) & = \depth_\cT(\tilde{P}) +2 \leq 2 \ceil{ \log_2(1/v(\tilde{P}))} +2  \leq  2 \ceil{ \log_2(1/2v(V))} +2 
    \\& \leq  2 \ceil{ \log_2(1/v(V)) - 1} +2  \leq  2 \ceil{ \log_2(1/v(V))}.
 \end{align*}
 This concludes the proof.
\end{proof}

\technicalqueryone*
\begin{proof}
Looking at it, we see that 
\begin{align*}
      & -  \sum_{i=1}^{n-1}  \sum_{y\in\cY}p(y)\E_{(Y_j)}\left[1_{\{|\frac{N_{y,i}-ip(y)}{ip(y)}|\geq \frac{1}{2}\}}\ind{N_{y,i} \geq 1}\log_2\left(\frac{N_{y,i}}{i}\right)\right] \\&
      \leq \sum_{i=1}^{n-1}  \sum_{y\in\cY}p(y) \bP\left(\left|\frac{N_{y,i}-ip(y)}{ip(y)}\right|\geq \frac{1}{2}\right)\log_2\left(i\right)
\end{align*}
Using a simplification of Chernoff's bound for Bernoulli variables \citep{Hoeffding1963}, we get the following bound  
$$\bP\left(\left|\frac{N_{y,i}-ip(y)}{ip(y)}\right|\geq \frac{1}{2}\right)\leq 2\exp(-\frac{ip(y)}{10}).$$
Hence
\begin{align*}
      \sum_{i=1}^{n-1}  \sum_{y\in\cY}p(y) \bP\left(\left|\frac{N_{y,i}-ip(y)}{ip(y)}\right|\geq \frac{1}{2}\right)\log_2\left(i\right) 
	  &\leq \sum_{y\in\cY}p(y) \log_2(n)\sum_{i=1}^{n-1} 2\sqrt{2}\exp(-\frac{ip(y)}{10})
      \\&
      \leq 2 \log_2(n)\sum_{y\in\cY} \frac{p(y)}{1-e^{-\frac{p(y)}{10}}}
      \leq 22 m \log_2(n),
\end{align*}
where the last inequality comes from the fact that $p\mapsto p / (1-e^{-\frac{p}{10}})$ is upper bounded by $1$ for $p\in[0,1]$.
Note that the constant could be improved, but not dramatically.
\end{proof}

\technicalquerytwo*
\begin{proof}
We will develop the logarithm up to the second order --higher orders do not easily yield better bounds.
\begin{align*}
      & - \ln(2)\E_{(Y_j)}\left[1_{\{|\frac{N_{y,i}-ip(y)}{ip(y)}|<\frac{1}{2}\}}\log_2\left(\frac{N_{y,i}}{ip(y)}\right)  \right] 
      \\&= - \ln(2)\E_{(Y_j)}\left[1_{\{|\frac{N_{y,i}-ip(y)}{ip(y)}|<\frac{1}{2}\}}\log_2\left(1+\frac{N_{y,i}-ip(y)}{ip(y)}\right)  \right] \\&
      = - \E_{(Y_j)}\left[1_{\{|\frac{N_{y,i}-ip(y)}{ip(y)}|<\frac{1}{2}\}}\left( 
      \frac{N_{y,i}-ip(y)}{ip(y)}-\frac{1}{2(1+\xi)^2}\left(\frac{N_{y,i}-ip(y)}{ip(y)}\right)^2
      \right)  \right]
\end{align*}
where $\xi$ is a random variable (inheriting its randomness from the $(Y_j)$) that belongs to the interval $\left[0, \frac{N_{y,i}-ip(y)}{ip(y)}\right]$ if $\frac{N_{y,i}-ip(y)}{ip(y)}\geq 0$ (respectively $\left[ \frac{N_{y,i}-ip(y)}{ip(y)}, 0\right]$ if $\frac{N_{y,i}-ip(y)}{ip(y)}\leq 0$) and depends on the value of $N_{y,i}$. Note that the Taylor expansion of $x\mapsto \ln(1+x)$ is valid because $\left|\frac{N_{y,i}-ip(y)}{ip(y)}\right|<\frac{1}{2}$.
Now
\begin{align*}
    &1_{\{|\frac{N_{y,i}-ip(y)}{ip(y)}|<\frac{1}{2}\}}\left(- 
      \frac{N_{y,i}-ip(y)}{ip(y)}+\frac{1}{2(1+\xi)^2}\left(\frac{N_{y,i}-ip(y)}{ip(y)}\right)^2
      \right)  \\&
      < 1_{\{|\frac{N_{y,i}-ip(y)}{ip(y)}|<\frac{1}{2}\}}\left(- 
      \frac{N_{y,i}-ip(y)}{ip(y)}+2\left(\frac{N_{y,i}-ip(y)}{ip(y)}\right)^2
      \right)\\&
      \leq
      2\left(\frac{N_{y,i}-ip(y)}{ip(y)+1}\right)^2 - \frac{N_{y,i}-ip(y)}{ip(y)} + 1_{\{|\frac{N_{y,i}-ip(y)}{ip(y)}|\geq\frac{1}{2}\}}\left( 
      \frac{N_{y,i}-ip(y)}{ip(y)}\right) \\&
      \leq 2\left(\frac{N_{y,i}-ip(y)}{ip(y)}\right)^2 
      - \frac{N_{y,i}-ip(y)}{ip(y)} 
      +  1_{\{|\frac{N_{y,i}-ip(y)}{ip(y)}|\geq\frac{1}{2}\}}\, 2\left( 
      \frac{N_{y,i}-ip(y)}{ip(y)}\right)^2\\&
      \leq 4\left(\frac{N_{y,i}-ip(y)}{ip(y)}\right)^2 
      - \frac{N_{y,i}-ip(y)}{ip(y)} 
\end{align*}
where the first inequality is valid because $1+\xi> \frac{1}{2}$.
The last two inequalities are used to avoid computing means of truncated binomials, together with the observations that directly bounding $\sum_{y,i}p(y)\E[ 1_{\{|\frac{N_{y,i}-ip(y)}{ip(y)}|\geq\frac{1}{2}\}}\left(\frac{N_{y,i}-ip(y)}{ip(y)}\right)]$ with a bound on $\bP(|\frac{N_{y,i}-ip(y)}{ip(y)}|\geq\frac{1}{2})$ does not yield a good dependence in $p(y)$, hence the rather crude upper bound with the square.
Using the standard formulas for the moments of binomial variables, we get that
\begin{align*}
    &\sum_{i=1}^{n-1}  \sum_{y\in\cY}p(y) \E_{(Y_j)}\left[
     4\left(\frac{N_{y,i}-ip(y)}{ip(y)}\right)^2 
      - \frac{N_{y,i}-ip(y)}{ip(y)} 
     \right]  \\&
      =
      \sum_{i=1}^{n-1}  \sum_{y\in\cY}
      \frac{4ip(y)^2(1-p(y))}{(ip(y))^2} - 0 \leq \sum_{i=1}^{n-1}  \sum_{y\in\cY}
      \frac{4}{i} \leq 4 m (\ln(n)+1).
\end{align*}
This proves the lemma.
\end{proof}

\subsection{Additional Proofs for Section \ref{sec:truncated}}
\label{app:truncated}

\trunctechnicalone*
\begin{proof}
  Consider the crude union bound
  \begin{align*}
   \bP({}^c\cA) 
	   &\leq \bP\left( \exists S \in 2^\cY \text{ s.t. }  |\hat p_{r-1}(S) -p(S)| >\frac{\epsilon_{r-1}-\epsilon_{r}}{4}\right)
	   \\&\qquad+ \bP\left( \exists S \in 2^\cY \text{ s.t. }  |\hat p_{r}(S) -p(S)|> \frac{\epsilon_{r-1}-\epsilon_{r}}{4}\right).
  \end{align*}
  Summing over all possible sets $S$ and using Hoeffding's inequality, we find the following (rather rough) bounds:
  \[\bP\left( \exists S \in 2^\cY \text{ s.t. }  |\hat p_{r-1}(S) -p(S)|> \frac{\epsilon_{r-1}-\epsilon_{r}}{4}\right)\leq 2^m \exp\left(-\frac{n_{r}}{2}\left(\frac{\epsilon_{r-1}-\epsilon_r}{4}\right)^2\right)\]
  and
  \[ \bP\left( \exists S \in 2^\cY \text{ s.t. }  |\hat p_{r}(S) -p(S)| >\frac{\epsilon_{r-1}-\epsilon_{r}}{4}\right) \leq 2^m \exp\left(-\frac{n_{r-1}}{2}\left(\frac{\epsilon_{r-1}-\epsilon_r}{4}\right)^2\right),\]
  Hence the lemma.
\end{proof}

\truncatedtechnicaltwo*
\begin{proof}
Dissociating the number of queries when $\cA$ holds and when it does not, together with Equation \eqref{eq:bound_on_event_Ac}, the expected number of queries needed for round $r\geq 2$ satisfies
\begin{align*}
    \E\left[T_r\right] & \leq 
    n_r \left( 2\ceil{\log_2\left(\frac{4}{p(y_1)}\right)} + 
    m \bP({}^c\cA)  \right) \\
    &\leq  n_r \left(  2\ceil{\log_2\left(\frac{4}{p(y_1)}\right)} + 
    m 2^{m+1}  \exp\left(-\frac{n_{r-1}}{2}\left(\frac{\epsilon_{r-1}-\epsilon_r}{4}\right)^2\right) \right)
    \\
    &\leq  2^r \left(2\ceil{\log_2\left(\frac{4}{p(y_1)}\right)}+ 
    m2^{m+1} \exp\left(-2^{r-6} \left(\frac{2}{3}\right)^{r-1} \left(\epsilon_{0}-\epsilon_1\right)^2\right) \right)
    \\
    &\leq  2^r \left( 2\ceil{\log_2\left(\frac{4}{p(y_1)}\right)} + 
    m 2^{m+1}  \exp\left(-\left(\frac{4}{3}\right)^r  \frac{C}{m^2} 
    \right) \right)
\end{align*}
for some constant $C>0$.
In the special case $r=1$, we need at most 
\[
    \E\left[T_1\right] \leq n_1 m = 2m
\]
queries to complete the round.
\end{proof}

\subsection{Additional Proofs for Section \ref{sec:elim}}
\label{app:elim}

Let us first prove Proposition~\ref{prop:comparison_Delta}, which we restate here.

\comparisondelta*

\begin{proof}
    $(\sqrt{p(y_1)} + \sqrt{p(y_i)})^2 \leq 2(p(y_i) + p(y_1))$, together with $\ln(1+x) \leq x$, imply
    \begin{align*}
        \Delta_i^2 
        = -\ln\parend[\Big]{1 - \frac{(p(y_1)-p(y_i))^2}{(\sqrt{p(y_1)}+\sqrt{p(y_i)})^2}}
        \geq \frac{(p(y_1)-p(y_i))^2}{2(p(y_1) + p(y_i))}
        \geq \frac{(p(y_1)-p(y_i))^2}{4p(y_1)}.
    \end{align*}
    For the other side, note that $\ln$ is always above its cords, hence for $x \in [a, b]$, 
    \[
        \ln(1-x) \geq \frac{(x-a)}{b - a} \paren{\ln(1-b) - \ln(1-a)} + \ln(1-a).
    \]
    Applied to $a = 0$ and $b = p(y_1)$, we get that
    \[
        \forall\,x\in[0, p(y_1)], \qquad \ln(1-x) \geq \frac{x\ln(1-p(y_1))}{p(y_1)}.
    \]
    As a consequence,
    \[
        \Delta_i^2 
       \leq \frac{-(p(y_1) - p(y_i))^2}{(\sqrt{p(y_1)}+\sqrt{p(y_i)})^2}\frac{\ln(1-p(y_1))}{p(y_1)}
       \leq \frac{-(p(y_1) - p(y_i))^2}{p(y_1)}\frac{\ln(1-p(y_1))}{p(y_1)}.
    \]
    This completes the proof.
\end{proof}

We now recall the following classical concentration result.

\begin{lemma}[Chernoff's bound]
    \label{lem:chern-mul}
    Let $(X_i)_{i=1}^n$ be independent Bernoulli variables with mean $p$. For any $\lambda \geq 0$,
   \[
       \bP\left(\frac{1}{n}\sum_{i\in[n]} X_i > (1+\lambda) p\right) \leq \exp\paren{-\frac{n\lambda^2 p}{2+\lambda}}.
   \] 
   Similarly, for $\lambda \in (0, 1)$
   \[
       \bP\left(\frac{1}{n}\sum_{i\in[n]} X_i < (1-\lambda) p\right) \leq \exp\paren{-\frac{n\lambda^2 p}{2}}.
   \] 
\end{lemma}
\begin{proof}
    Those are classical relaxations of some results that can be found in \citet{Hoeffding1963}.
\end{proof}

\modeestone*
\begin{proof}
    This follows from Lemma \ref{lem:chern-mul} since
    \begin{align*}
        \bP(\hat p_r(y) > cp(y)) &= \bP(\hat p_r(y) > (1 + (c-1)) p(y)) 
        \leq \exp\paren{\frac{-r (c-1)^2 p(y)}{2 + (c-1)}} 
      \\&= \exp\paren{\frac{-r (c-1)^2 p(y)}{c+1}} 
        \leq \exp(\ln(1/\delta) = \delta,
    \end{align*}
    and
    \begin{align*}
        \bP(\hat p_r(y) > c^{-1} p(y)) &= \bP(\hat p_r(y) > (1 - (1 - c^{-1})) p(y)) 
        \\&\leq \exp\paren{\frac{-r (1-c^{-1})^2 p(y)}{2}} 
        \leq \exp(\ln(1/\delta) )= \delta.
    \end{align*}
    This explains the result of the Lemma.
\end{proof}

\modeesttwo*
\begin{proof}
    We can upper bound the probability of $\cA_r$ not happening with
    \begin{align*}
        \bP({}^c\cA_r) 
        &= \bP(r\hat p_r(\hat y_r) > c\ln(1/\delta)) 
        = \bP(r \max_{y\in\cY} \hat p_r(y) > c\ln(1/\delta))
        \\&\leq \sum_{y\in\cY} \bP(r \hat p_r(y) > c\ln(1/\delta))
        \leq m\bP(r \hat p_r(y_1) > c\ln(1/\delta)).
    \end{align*}
    Let us set
    \[
        \alpha_r = \frac{c \ln(1/\delta)}{r p(y_1)} - 1 , \qquad\text{i.e.,}\qquad r = \frac{c \ln(1/\delta)}{p(y_1) (1+\alpha_r)}.
    \]
    Note that $\alpha_r\geq 0$ when $r \leq \frac{2c^2 - 2c}{2c + 1 + \sqrt{1+8c}} \frac{1}{p(y_1)} \ln(1/\delta)$.
    Using Lemma \ref{lem:chern-mul} leads to
    \begin{align*}
        \bP({}^c\cA_r)
        &\leq m\bP\parend[\Big]{\hat p_r(y_1) - p(y_1) > c\ln(1/\delta) / r - p(y_1)}
        = m\bP\parend[\Big]{\hat p_r(y_1) - p(y_1) > \alpha_r p(y_1)}
        \\&\leq m\exp\paren{\frac{- r \alpha_r^2 p(y_1)}{2 + \alpha_r}}
        = m\exp\paren{\frac{- c \ln(1/\delta) \alpha_r^2}{(1+\alpha_r)(2 + \alpha_r)}}.
    \end{align*}
    We check that
    \begin{align*}
        &r \leq \frac{2c^2 - 2c}{2c + 1 + \sqrt{1+8c}}\frac{1}{p(y_1)} \ln(1/\delta)
        \quad\Leftrightarrow\quad
        \frac{c}{1+\alpha_r} \leq \frac{2c^2 - 2c}{2c + 1 + \sqrt{1+8c}}
        \\&\quad\Leftrightarrow\quad
        \alpha_r \geq \frac{3 + \sqrt{1 + 8c}}{2(c-1)}
        \quad\Rightarrow\quad
        (c-1)\alpha_r^2 - 3\alpha_r - 2 \geq 0
        \quad\Leftrightarrow\quad
        \frac{c \alpha_r^2}{(1+\alpha_r)(2+\alpha_r)} \geq 1,
    \end{align*}
    which allows us to conclude that $\bP({}^c\cA_r)\leq m\delta$.
    This explains the condition on $r$ stated in the Lemma.

    We can similarly lower bound the probability of $\cA_r$ happening with 
    \begin{align*}
        \bP(\cA_r) 
        = \bP(r\hat p_r(\hat y_r) < c\ln(1/\delta)) 
        = \bP(r \max_{y\in\cY} \hat p_r(y) < c\ln(1/\delta))
        \leq \bP(r \hat p_r(y_1) < c\ln(1/\delta)).
    \end{align*}
    This time we set
    \[
        \alpha_r = 1 - \frac{c \ln(1/\delta)}{r p(y_1)}, \qquad\text{i.e.,}\qquad r = \frac{c \ln(1/\delta)}{p(y_1) (1-\alpha_r)}.
    \]
    One can check that  if $r \geq \frac{c^2}{c + 1 - \sqrt{1+2c}} \frac{1}{p(y_1)} \ln(1/\delta)$, then $\alpha_r \in (0,1)$. 
    Using Lemma \ref{lem:chern-mul},
    \begin{align*}
        \bP(\cA_r) 
        &\leq \bP\parend[\Big]{\hat p_r(y_1) - p(y_1) < c\ln(1/\delta) / r - p(y_1)}
        = \bP\parend[\Big]{\hat p_r(y_1) - p(y_1) < -\alpha_r p(y_1)}
        \\&\leq \exp\paren{\frac{- r \alpha_r^2 p(y_1)}{2}}
        \leq \exp\paren{\frac{- c \ln(1/\delta) \alpha_r^2}{2(1-\alpha_r)}}.
    \end{align*}
    We check that
    \begin{align*}
        &r \geq \frac{c^2}{c + 1 - \sqrt{1+2c}} \frac{1}{p(y_1)} \ln(1/\delta)
        \qquad\Leftrightarrow\qquad
        \frac{c}{1-\alpha_r} \geq \frac{c^2}{c + 1 - \sqrt{1+2c}}
        \\&\quad\Leftrightarrow\qquad
        \alpha_r \geq \frac{\sqrt{1 + 2c} - 1}{c}
        \qquad\Rightarrow\qquad
        c\alpha_r^2 + 2\alpha_r - 2 \geq 0
        \qquad\Leftrightarrow\qquad
        \frac{c \alpha_r^2}{2(1-\alpha_r)} \geq 1.
    \end{align*}
    We conclude that $\bP(\cA_r)\leq \delta$ for the $r$ defined in the Lemma.
\end{proof}

\concentrationelim*
\begin{proof}
 We can apply the second statement of Lemma \ref{lem:mode-est}  with $c=2$ to $y_1$ to see that for any $r\geq \frac{4}{p(y_1)} \ln(1/\delta_r) =  \frac{2^2}{(2-1)^2} \frac{1}{p(y_1)} \ln(1/\delta_r) $, 
we have 
\[\bP\paren{\hat p_r(y_1) <  p(y_1)/2} \leq \delta_r.\]
Hence
\[\bP({}^c A_1)\leq \sum_{r\geq  \frac{4}{p(y_1)} \ln(1/\delta_r)} \delta_r \leq \sum_{r\geq 1} \delta_r \leq \delta/6. \]

Similarly, we can apply the first statement of Lemma \ref{lem:mode-est}  with $c=2$ to $y_1$ to see that for any $r\geq\frac{4}{p(y_1)} \ln(1/\delta_r) \geq  \frac{2 +1}{(2-1)^2} \frac{1}{p(y_1)} \ln(1/\delta_r) $ and any $y\in\cY$, 
we have 
\[\bP\paren{\hat p_r(y) > 2 p(y_1)}\leq\bP\paren{\hat p_r(y_1) > 2 p(y_1)} \leq \delta_r.\]
Hence, using a union bound over all $y\in \cY$, we find that 
\begin{equation} 
     \bP\paren{\hat p_r(\hat y) <2 p(y_1)} \leq  m\delta_r \qquad \forall r\geq  \frac{4}{p(y_1)} \ln(1/\delta_r),
\end{equation}
thus 
\[\bP({}^c A_2)\leq \sum_{r\geq  \frac{4}{p(y_1)} \ln(1/\delta_r)} m\delta_r \leq \sum_{r\geq 1} m\delta_r \leq \delta/6. \]

We can now consider $A_3$. Let  $i\in \{1,\ldots,m\}$.
Applying\footnote{
    Note that we can also use Hoeffding inequality, which leads to
     \[
        \bP\paren{\hat p_r(y_i) \geq p(y_i) + \tilde{\sigma}_r} \leq \exp(-r\tilde{\sigma}_r^2/2)
    \]
    which is tighter when $p(y_i) > 1 - \tilde{\sigma}_r / 2$.
}  Lemma \ref{lem:chern-mul}, we get that if $r\geq \frac{4}{p(y_1)} \ln(1/\delta_r)$ (hence $\sqrt{\frac{4\ln(1/\delta_r)}{p(y_1)r}}\in (0,1)$), then 
\begin{align*}
\bP\left(\hat p_r(y_i) < p(y_i) - \sqrt{\frac{3  p(y_1)\ln(1/\delta_r)}{r}}\right) & = \bP\left(\hat p_r(y_i) < p(y_i) \left(1 - \sqrt{\frac{3p(y_1)\ln(1/\delta_r)}{p(y_i)^2r}}\right)\right) 
\\&\leq \exp\paren{-\frac{r p(y_i) 3p(y_1) \ln(1/\delta_r)}{2p(y_i)^2r}} 
\\& = \exp\paren{- \ln(1/\delta_r)\frac{3}{2}} 
\leq 
\delta_r.    
\end{align*}

Using this time the second inequality of Lemma \ref{lem:chern-mul}, and writing $\tilde{\sigma}_r = \sqrt{\frac{3 p(y_1)\ln(1/\delta_r)}{r}}$, we see that for any $r\geq 1$ we have 
\begin{align*}
&\bP\left(\hat p_r(y_i) > p(y_i) + \sqrt{\frac{3p(y_1)\ln(1/\delta_r)}{r}}\right) = \bP\left(\hat p_r(y_i) > p(y_i)\left(1 + \frac{\tilde{\sigma}_r}{p(y_i)}\right)\right)
\\&\leq \exp\paren{-\frac{r p(y_i)\frac{\tilde{\sigma}_r^2}{p(y_i)^2} }{2+\frac{\tilde{\sigma}_r}{p(y_i)}}} = \exp\paren{-\frac{r \tilde{\sigma}_r^2 }{2p(y_i)+\tilde{\sigma}_r}} 
\\&\leq \exp\paren{-\frac{r \tilde{\sigma}_r^2 }{2p(y_1)+\tilde{\sigma}_r}}.
\end{align*}
Furthermore, if $r\geq \frac{4}{p(y_1)} \ln(1/\delta_r)$,  then $\tilde{\sigma}_r = \sqrt{\frac{3 p(y_1)\ln(1/\delta_r)}{r}}\leq  \sqrt{\frac{3}{4}} p(y_1)$  
and consequently 
\[ \exp\paren{-\frac{r \tilde{\sigma}_r^2 }{2p(y_1)+\tilde{\sigma}_r}} 
\leq 
\exp\paren{-\frac{r \tilde{\sigma}_r^2 }{p(y_1)(2+ \sqrt{\frac{3}{4}} )}} 
\leq
\exp\paren{- \ln(1/\delta_r) \frac{3 }{2+ \sqrt{\frac{3}{4}} }} 
\leq \delta_r.\]

Combining those two bounds, we find that
\[\bP({}^c A_3)\leq \sum_{r\geq  \frac{4}{p(y_1)}\ln(1/\delta_r)} \sum_{i=1}^m \bP\left( |\hat p_r(y_i) - p(y_i) |> \sqrt{\frac{3p(y_1)\ln(1/\delta_r)}{r}}\right)  \leq   \sum_{r\geq 1}2 m\delta_r \leq \delta/3. \]

Let us now consider $A_4$.
We can apply Lemma \ref{lem:mode-est-crit} (with $c=24$) to see that if 
\[
    r\leq \frac{4}{p(y_1)} \ln(1/\delta_r) \leq \frac{2\cdot 24^2 - 24}{2\cdot 24 + 1 + \sqrt{1+8\cdot 24}} \frac{1}{p(y_1)}\ln(1/\delta_r),
\]
then 
\begin{align*}
 \bP(\sigma_r < \hat p_r(\hat y)) & =
 \bP(r\hat p_r(\hat y)> 24\ln(1/\delta_r)) 
 \leq \bP(r\hat p_r(\hat y_{all})> 24\ln(1/\delta_r))
 \leq
 m\delta_r.     
\end{align*}
Hence, using an union bound over all $r\leq \frac{4}{p(y_1)} \ln(1/\delta_r)$, we get that
\[\bP({}^c A_4)\leq \sum_{r\leq \frac{4}{p(y_1)} \ln(1/\delta_r)} m\delta_r  \leq \sum_{r\geq 1} m\delta_r \leq \delta/6. \]  
This ends the proof
\end{proof}

\eliminationtech*
\begin{proof}
For $r\geq 1$  such that Algorithm \ref{algo:elimination} has not yet terminated at the start of round $r$, let $S_r$ be the set of all classes that have been eliminated in the previous rounds (it is a random set), and let $
Q_r$ be the number of queries necessary to identify $Y_r$ according to the partition $  \{S_r\}\cup \brace{\{y\}\midvert y\in \cY \backslash S_r}$ using first the query $\ind{Y_r\in S_r}$, then a Huffman tree adapted to the (renormalized) empirical distribution $ \frac{\hat p_r}{1-\hat p_r (S_r)}$ on $\cY \backslash S_r$ if $Y_r \not \in S_r$. 
Hence $Q_r = 1$ if $Y_r\in S_r$, and as in Subsection \ref{proof:huffman},
\[Q_r \leq 1+  2 + 2\cdot\ind{\{\hat p_r(y) \neq 0\}} \log_2\left(\frac{1-\hat p_r (S_r)}{\hat p_r(y)}\right) + \ind{\{\hat p_r(y) = 0\}}  m\]
if $Y_r = y \in \cY\backslash S_r$.

For $y\in \cY\backslash\{y_1\}$, let  $r(y)$ be as above the smallest $r_0\in \bN$ such that
\[ r>  108 \frac{1}{(p(y_1)-p(y))^2}  p(y_1)\ln(1/\delta_r)\]
for all $r\geq r_0$, with the additional convention that $r(y_1) := r(y_2)$.
As seen in Lemma \ref{lem:elimination_time}, if $A_1,A_2,A_3$ and $A_4$ hold, then the class $y$ necessarily belongs to $S_r$ at the start of round $r$ as soon as $r> r(y)$, hence $1_{\{y\not \in S_r\}} \leq 1_{\{ r\leq r(y) \}}$.
Thus the conditional expectation of $Q_r$ with respect to the events $\{A_i\}_{i=1}^4$ satisfies
\begin{align*}
  \E \left[ Q_r | \{A_i\}_{i=1}^4 \right] & \leq 
  3 + \E \left[ \ind{\{\hat p_r(y) = 0\}}  m\midvert \{A_i\}_{i=1}^4 \right] 
  \\& + 2\E \left[  \ind{\{\hat p_r(y) \neq  0\}} 1_{\{y\not \in S_r\}} \log_2\left( \frac{1-\hat p_r(S_r)}{\hat p_r(y)}\right) \midvert \{A_i\}_{i=1}^4 \right] 
  \\& \leq 3 +  \E \left[ \ind{\{\hat p_r(y) = 0\}}  m\midvert \{A_i\}_{i=1}^4 \right]
  \\& +  2\E \left[ \ind{\{\hat p_r(y) \neq  0\}} 1_{\{ r\leq r(y) \}} \log_2\left( \frac{1}{\hat p_r(y)}\right) \midvert \{A_i\}_{i=1}^4 \right].
\end{align*}
Now let $\Omega$ be the probability space of all possible outcomes, and let $S(\omega)$ denote the (positive) random variable $\ind{\{\hat p_r(y) \neq  0\}} 1_{\{ r\leq r(y) \}} \log_2\left( \frac{1}{\hat p_r(y)}\right) $.
Then 
\begin{align*}
    \E \left[S(\omega) \midvert \{A_i\}_{i=1}^4 \right] 
    &= \frac{1}{\bP(\cap_{i=1}^4 A_i)} \int_{\cap_{i=1}^4 A_i} S(\omega)d_\omega 
    \\&\leq  \frac{1}{\bP(\cap_{i=1}^4 A_i)} \int_{\Omega} S(\omega)d_\omega
     = \frac{1}{\bP(\cap_{i=1}^4 A_i)} \E \left[S(\omega) \right].
\end{align*}
As $\bP(\cap_{i=1}^4 A_i)\geq 1-\delta$, this means that 
\begin{align*}
&\E \left[  \ind{\{\hat p_r(y)\neq  0\}} 1_{\{ r\leq r(y) \}} \log_2\left( \frac{1}{\hat p_r(y)}\right) \midvert \{A_i\}_{i=1}^4 \right] 
\\& \leq 
\frac{1}{(1-\delta)} \E \left[ \ind{\{\hat p_r(y) \neq  0\}}  1_{\{ r\leq r(y) \}}\log_2\left( \frac{1}{\hat p_r(y)}\right)  \right]
\\&  = \frac{1}{(1-\delta)} \sum_{y\in \cY} p(y) 1_{\{ r\leq r(y) \}} \E \left[ \ind{\{\hat p_r(y) \neq  0\}}  \log_2\left( \frac{1}{\hat p_r(y)}\right)  \right]
\end{align*}
We have already shown in Subsection 
\ref{proof:huffman} that for any $y$ such that $p(y)\neq 0$,
\[\E \left[ \ind{\{\hat p_r(y) \neq  0\}}  \log_2\left( \frac{1}{\hat p_r(y)}\right)  \right] \leq  \log_2(p(y)) + \frac{4}{p(y)r\ln(2)} + 2\exp\left(-\frac{rp(y)}{10}\right)\log_2(r).\]
As we have assumed that $A_1,A_2,A_3$ and $A_4$ hold, we know that  Algorithm \ref{algo:elimination} terminates at the end of some round $R \leq r(y_2)$. 
Let $T = \sum_{r=1}^{R} Q_r$ be the number of queries used to identify those $R$ samples.
Combining the results above, we find that
\begin{align*}
  &\E \left[ T_R | \{A_i\}_{i=1}^4 \right] 
  \leq 3R
  + \sum_{r=1}^R  \E \left[ \ind{\{\hat p_r(y) = 0\}}  m\midvert \{A_i\}_{i=1}^4 \right] 
 \\& +   2 \sum_{y\in \cY}\sum_{r=1}^{r(y)}\frac{p(y)}{(1-\delta)} \left( \abs{\log_2(p(y))} + \frac{4}{p(y)r\ln(2)}  + 2\exp\left(-\frac{rp(y)}{10}\right)\log_2(r)\right),
\end{align*}
where the terms in the sum are understood to be $0$ for those classes $y$ such that $p(y)=0$.
As $\ind{\{\hat p_r(y) = 0\}}$ can only be non-zero for a single index $r$ for each class $y$, 
\[\sum_{r=1}^R  \E \left[ \ind{\{\hat p_r(y) = 0\}}  m\midvert \{A_i\}_{i=1}^4 \right] \leq m^2.\]
Furthermore,
\[ \sum_{y\in \cY} \sum_{r=1}^{r(y)}\frac{4}{r\ln(2)} \leq  \frac{4m(\ln(R)+1)}{\ln(2)} \]
and 
\begin{align*}
\sum_{y\in \cY} \sum_{r=1}^{r(y)}p(y)  2\exp\left(-\frac{rp(y)}{10}\right)\log_2(r) 
& \leq 
2 \sum_{y\in \cY} \frac{p(y)}{1-e^{\frac{p(y)}{10}}} \log_2(R) \leq  22m \log_2(R) 
\end{align*}
due to $p\mapsto p / (1-e^{-\frac{p}{10}})$ being upper bounded by $11$ for $p\in[0,1]$.
Finally, using the fact that $ r(y_i) = 108\tilde{\Delta}_i^{-2}  p(y_1) \ln(1/\delta) + o(\ln(1/\delta))$ (as seen in Equation 
\ref{eq:elimination_time_in_proof}), we see that
\begin{align*}
\sum_{y\in \cY}\sum_{r=1}^{r(y)}p(y) \abs{\log_2(p(y))} 
&=\sum_{i=1}^m p(y_i) \abs{\log_2(p(y_i))}r(y_i) 
\\&= 108\sum_{i=1}^m p(y_i) \abs{\log_2(p(y_i))}\tilde{\Delta}_i^{-2}  p(y_1) \ln(1/\delta) + o(\ln(1/\delta)).     
\end{align*}
Thus 
\begin{align*}
  \E \left[ T_R | \{A_i\}_{i=1}^4 \right] & 
  \leq 3r(y_2) + \frac{216}{1-\delta}\sum_{i=1}^m p(y_i) \abs{\log_2(p(y_i))}\tilde{\Delta}_i^{-2}  p(y_1) \ln(1/\delta) + o(\ln(1/\delta))
 \\&\qquad + m^2 + 
  \frac{2}{1-\delta} \left( \frac{4m(\ln(r_2)+1)}{\ln(2)} +  22m \log_2(r_2) \right)
\\& = 324\tilde{\Delta}_2^{-2}  p(y_1) \ln(1/\delta) + 216\sum_{i=1}^m p(y_i) \abs{\log_2(p(y_i))}\tilde{\Delta}_i^{-2}  p(y_1) \ln(1/\delta) 
 \\&\qquad +o(\ln(1/\delta))
\end{align*}
This completes the proof.
\end{proof}

\subsection{Additional Proofs for Section \ref{sec:set-elim}}
\label{app:set-elim}

\setelimquery*
\begin{proof}
Under $(A_j)_{j\in[4]}$, we have the bound $ 1_{\{Y_{i,r}\not \in S_r\}}\leq  1_{\{r \leq r(Y_{i,r}) \}}$.
Let us add the convention that $r(y_1):= r(y_2)$.
We see that the algorithm necessarily terminates at the end of some round $R\leq r(y_1)$.
Going back to Equation \eqref{eq:set-elimination-first-bound-T}, we deduce that the conditional expectation of $T_r$ (for $2\leq r\leq R$) with respect to the events $\{A_i\}_{i=1}^4$ satisfies
\begin{align*}
  \E \bracket{ T_r \midvert (A_i)_{i\in[4]} } & \leq 
   \E \bracketd[\Big]{  \sum_{i=1}^{n_r} 1 + 1_{\{ r \leq r(Y_{i,r})\}} \left( 2\ceil{\log_2\left(\frac{10}{p(y_1)}\right)} + 1_{{}^c B_r} m \right)\Big\vert (A_i)_{i\in[4]} }
\end{align*}
Using the same simple integral argument as in the proof of Theorem \ref{thm:elimination}, we further see that as $\bP(\cap_{i=1}^4 A_i)\geq 1-\delta$, then
\begin{align*}
& \E \left[  \sum_{i=1}^{n_r} 1 + 1_{\{ r \leq r(Y_{i,r})\}} \left( 2\ceil{\log_2\left(\frac{10}{p(y_1)}\right)} + 1_{{}^c B_r} m \right)\midvert \{A_i\}_{i=1}^4 \right] 
\\ & \leq \frac{1}{1-\delta } \E \left[  \sum_{i=1}^{n_r} 1 + 1_{\{ r \leq r(Y_{i,r})\}} \left(2 \ceil{\log_2\left(\frac{10}{p(y_1)}\right)} + 1_{{}^c B_r} m\right) \right] .
\end{align*}
We can further write
\begin{align*}
   & \E \left[  \sum_{i=1}^{n_r} 1 + 1_{\{ r \leq r(Y_{i,r})\}} \left( 2\ceil{\log_2\left(\frac{10}{p(y_1)}\right)} + 1_{{}^c B_r} m \right) \right] 
   \\& \leq n_r +  n_r \bP({}^c B_r)m + 2n_r \sum_{y\in\cY}p(y) 1_{\{ r \leq r(y)\}}  \ceil{\log_2\left(\frac{10}{p(y_1)}\right)}
\end{align*}

We can now bound the expectation of the total number $T$ of queries required before the algorithm terminates at the end of round $R$:
\begin{align*}
(1-\delta)\E \left[ T \right] & \leq n_1 m+  \sum_{r=2}^R n_r + \sum_{r=2}^R n_r \bP({}^c B_r)m
\\& + \sum_{r=2}^R 2n_r \sum_{y\in\cY}p(y) 1_{\{ r \leq r(y)\}}  \ceil{\log_2\left(\frac{10}{p(y_1)}\right)}.
\end{align*}
Observe that $\sum_{r=2}^R n_r \leq 2n_R \leq 2n_{r(y_1)}$ and that 
\[n_1 m+   \sum_{r=2}^R n_r \bP({}^c B_r)m  \leq m(n_1 +   \sum_{r\in \bN} n_r \bP({}^c B_r)) \leq \tilde{C}\] 
for some constant $\tilde{C}>0$ independent from $\delta$ (using Equation \eqref{eq:bound_Brc}).
Furthermore,
\begin{align*}
    \sum_{r=2}^R 2n_r \sum_{y\in\cY}p(y) 1_{\{ r \leq r(y)\}}  \ceil{\log_2\left(\frac{10}{p(y_1)}\right)}  &\leq   \ceil{\log_2\left(\frac{10}{p(y_1)}\right)} \sum_{y\in\cY}  p(y) \sum_{r=2}^{r(y)} 2n_r   
    \\& \leq \ceil{\log_2\left(\frac{10}{p(y_1)}\right)} \sum_{y\in\cY}  p(y) 4n_{r(y)}.
\end{align*}
Combining those results and using Lemma \eqref{lem:bound-nry}, we finally see that
\begin{align*}
(1-\delta)\E \left[ T \right] & \leq \tilde{C} + 2 n_{r(y_1)}
 + \ceil{\log_2\left(\frac{10}{p(y_1)}\right)} \sum_{y\in\cY}  p(y) 4 n_{r(y_1)}
\\& \leq \tilde{C} + 432\frac{1}{(p(y_1)-p(y))^2} p(y_1) \ln(1/\delta_r) 
\\& + 864\sum_{y\in\cY \text{ s.t. } p(y)< p(y_1)/2}  p(y)  \frac{4}{p(y_1)^2} p(y_1)\ceil{\log_2\left(\frac{10}{p(y_1)}\right)} \ln(1/\delta_r)
\\& + 864\sum_{y\in\cY \text{ s.t. } p(y)\geq p(y_1)/2}  p(y) \frac{1}{(p(y_1)-p(y))^2} p(y_1)\ceil{\log_2\left(\frac{10}{p(y_1)}\right)}\ln(1/\delta_r),
\end{align*}
hence that 
\begin{align*}
&\E \left[ T \right] \leq 432\frac{1}{(p(y_1)-p(y_2))^2}  p(y_1) \ln(1/\delta) 
\\& + 864\sum_{y\in\cY \text{ s.t. } p(y)< p(y_1)/2}  p(y)  \frac{4}{p(y_1)^2} p(y_1) \ceil{\log_2\left(\frac{10}{p(y_1)}\right)}\ln(1/\delta)
\\& + 864\sum_{y\in\cY \text{ s.t. } p(y)\geq p(y_1)/2}  p(y) \frac{1}{(p(y_1)-p(y))^2} p(y_1)\ceil{\log_2\left(\frac{10}{p(y_1)}\right)}\ln(1/\delta) +o(\ln(1/\delta)).
\end{align*}
This ends the proof of the lemma.
\end{proof}

\end{appendix}

\begin{acks}[Acknowledgments]
The authors would like to thank Gilles Blanchard, Remi Jezequiel, Marc Jourdan, Nadav Merlis, and Karen Ullrich for fruitful discussions.
\end{acks}
\bibliographystyle{imsart-nameyear}
\bibliography{bibliography}

\end{document}